\documentclass{article}

\usepackage{microtype}
\usepackage{booktabs} % for professional tables
\usepackage{longtable}

\usepackage{hyperref}
\usepackage{enumitem}
\usepackage{amsmath,amssymb}
 \usepackage{graphicx}
\usepackage{caption}
\usepackage{subcaption}

\usepackage[accepted]{icml2026}

\usepackage{amsmath}
\usepackage{amssymb}
\usepackage{mathtools}
\usepackage{amsthm}
\usepackage{bbm}
\usepackage{pifont}
\usepackage{multirow}
\usepackage{algorithm}
\usepackage{algorithmic}
\usepackage[capitalize,noabbrev]{cleveref}
\usepackage{lipsum}

\theoremstyle{plain}
\newtheorem{theorem}{Theorem}[section]
\newtheorem{proposition}[theorem]{Proposition}

\theoremstyle{definition}

\theoremstyle{remark}

\newcommand{\f}{\texttt{f}}
\newcommand{\g}{\texttt{g}}

\newenvironment{sequation}{\begin{equation}\setlength\abovedisplayskip{2pt}\setlength\belowdisplayskip{2pt}}{\end{equation}}
\usepackage[textsize=tiny]{todonotes}

\newcommand{\model}{SW-DRSO}

\icmltitlerunning{Distributionally Robust Set Representation Learning Under Inference-Time Element Corruption}

\begin{document}

\twocolumn[
  % \icmltitle{Robust Set Representation Learning and Optimization under \\ Set-Structured Partial Evidence Degradation}
  % \icmltitle{Efficient Distributional Robustness Optimization for Sets Under Inference-time Element Corruption: A Latent Representation Synthesis Approach}
\icmltitle{Distributionally Robust Set Representation Learning Under Inference-Time Element Corruption}



  \icmlsetsymbol{Corresponding Author}{*}

  \begin{icmlauthorlist}
    \icmlauthor{Yankai Chen}{a,b}
    \icmlauthor{Hanrong Zhang}{c}
    \icmlauthor{Bowei He}{a,b,Corresponding Author}
    \icmlauthor{Philip S. Yu}{c}
    \icmlauthor{Xue (Steve) Liu}{a,b}
  \end{icmlauthorlist}

  \icmlaffiliation{a}{MBZUAI}
  \icmlaffiliation{b}{McGill University}
  \icmlaffiliation{c}{University of Illinois Chicago}

  \icmlcorrespondingauthor{Yankai Chen}{yankaichen@acm.org}
  \icmlcorrespondingauthor{Bowei He}{Bowei.He@mbzuai.ac.ae}

  \icmlkeywords{Machine Learning, ICML}

  \vskip 0.3in
]



\printAffiliationsAndNotice{*Corresponding author.}  % no special notice (required even if empty)

\begin{abstract}
Standard Set Representation Learning methods typically excel on curated data but often overlook the challenge of \textit{Inference-time Element Corruption}. 
This refers to scenarios where deployed models encounter element-level degradations, such as outliers or missing components, that may distort set representation and degrade performance.
We propose \model, a distributionally robust optimization framework tailored for sets. 
Rather than minimizing loss solely on observed training data, \model~optimizes a tractable surrogate of the
worst-case expected loss over a family of plausible inference-time variations. 
We introduce a barycentric adversary that approximates the intractable search over corrupted sets by a differentiable training-time optimization over simplex weights.
Extensive experiments across four tasks demonstrate that \model~effectively enhances robustness against corruption while maintaining high overall performance.

\end{abstract}

\section{Introduction}

Set Representation Learning (SRL) aims to encode a unordered collection of elements into a vector while preserving essential set-theoretic properties~\cite{zaheer2017deep, skianis2020rep}. 
This capability is crucial across diverse domains, from processing point clouds~\cite{qi2017pointnet} and modeling protein structures~\cite{naderializadeh2025aggregating}, to enabling batch retrieval and recommendation~\cite{li2021les3,zhang2022knowledge,hihpq,chen2023star,he2024interpretable}. 
Recent methods employ attention mechanisms~\cite{lee2019set}, optimal transport formulations~\cite{skianis2020rep, mialon2021trainable}, and sophisticated aggregation strategies~\cite{zhang2019fspool} to capture complex set structures while adhering to preserve permutation invariance and handle variable cardinality.

\begin{figure}[t]
    \begin{center}
    % \centerline{\includegraphics[width=\columnwidth]{figs/overview_corruption.pdf}}
    \centerline{\includegraphics[width=\columnwidth]{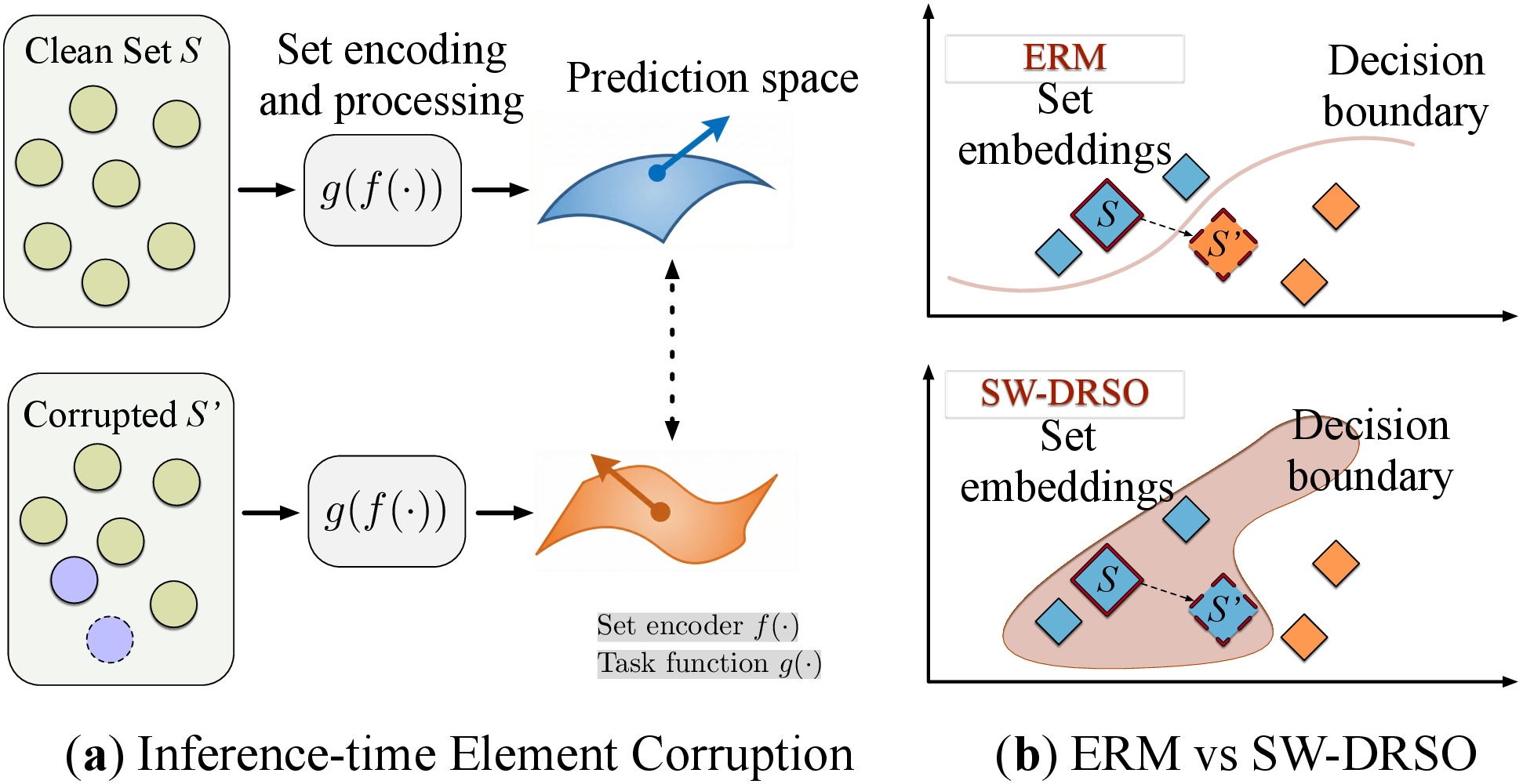}}
    \caption{Illustration of inference-time element corruption. (a) The corruption leads to variations in the prediction space. (b) While ERM performance degrades on the corrupted set $S'$, our \model~enhances robustness by optimizing the decision boundary.}
    \label{fig:overview}
    \end{center}
\end{figure}

While standard SRL methods excel on \textit{curated} data, they often overlook a robustness challenge in practical environments: \textit{Inference-time element corruption}. 
Specifically, it refers to scenarios where models trained on clean and complete sets, but are deployed on inputs with element-level degradation at inference time. 
Such corruption is typically sparse and localized: it affects only a small fraction of elements without altering the overall set theme, i.e., \textit{label-preserving}. 
Common examples include outlier points that appear in point clouds due to sensor errors~\cite{rusu2008towards}.
Despite their sparsity, as shown in Figure~\ref{fig:overview}(a), even a few corrupted elements may impact the set encoder and task processing to produce unreliable predictions. 
For instance, attention-based pooling is vulnerable to outliers that attract disproportionate weights, thereby dominating the global representation~\cite{zhou2022understanding}.

To handle such inference-time uncertainty, \textit{Distributionally Robust Optimization (DRO)} offers a principled framewor~\cite{sagawa2019distributionally,bertsimas2019adaptive}.
By minimizing the worst-case expected loss over a family of plausible distributions, DRO can hedge against unknown inference-tiem corruption patterns~\cite{rahimian2019distributionally}. 
This approach offers a natural extension beyond standard Empirical Risk Minimization (ERM), which optimizes performance only on nominal training distributions and may fail to generalize when inference-time inputs deviate from the training regime.

However, applying DRO directly to SRL problems faces fundamental computational challenges.
First, defining the ambiguity regions\footnote{The family of plausible corrupted distributions. We use ``region'' instead of the standard term ``ambiguity set'' to avoid confusion with the set-structured data itself.} is non-trivial for set-structured data.
Unlike continuous inputs where ambiguity regions are naturally constructed via norm-bounded perturbations~\cite{rahimian2022frameworks}, sets are discrete and combinatorial.
Second, even with a well-defined ambiguity region, solving the inner maximization in conventional DRO formulation to identify the worst cases is computationally prohibitive. 
This typically requires solving discrete combinatorial optimization with factorial complexity~\citep{gao2024wasserstein,kuhn2019wasserstein}, making the training process unscalable.

To address these challenges, we propose \textbf{SW-DRSO} (Sliced-Wasserstein Distributionally Robust Set Optimization), a scalable DRO framework designed for SRL. 
As shown in Figure~\ref{fig:overview}(b), it improves the robustness with better optimized decision boundary compared with conventional ERM.
Specifically, we first represent sets as empirical measures~\citep{muandet2012learning} and define the corruption region via the Sliced-Wasserstein (SW) metric~\citep{bonneel2015sliced,pswe,amirfourier}. 
This formulation enables flexible modeling of corrupted sets without requiring discrete enumeration, while the geometric properties of the SW metric facilitate set encoding and ambiguity region construction. 
Second, we introduce a data synthesis strategy that approximates the worst-case optimization through Wasserstein barycentric adversaries~\citep{agueh2011barycenters}. 
By synthesizing barycenters from local neighborhoods, we transform the intractable search over the combinatorial set space into efficient optimization over a \textit{low-dimensional probability simplex} (i.e., the convex hull of mixing weights). 
Our analysis shows that this differentiable parameterization approximates the inner supremum better than discrete optimization. 
To empirically validate \model, we extensively evaluate \model~across four diverse downstream tasks spanning varied data modalities and distinct corruption patterns, comparing it against state-of-the-art set representation baselines. 
% Codes are available via \url{https://anonymous.4open.science/r/drso-0C7A}.
In summary, our main contributions are three-fold: 
\begin{itemize}[leftmargin=*]
\item We propose \model~to enable effective distributionally robust optimization for Set Representation Learning.
\item We formulate and optimize a tractable ambiguity region through barycentric data synthesis, which converts the computationally prohibitive worst-case search into an efficient, differentiable optimization task.
\item Our results show that \model~consistently outperforms baselines, enhancing robustness against severe corruptions without compromising accuracy on clean data.
\end{itemize}

\section{Related Work}

\textbf{Set Representation Learning (SRL).}  
Representation learning aims to transform raw structured objects into compact embeddings that preserve task-relevant semantics for downstream prediction, retrieval, clustering, and decision making~\citep{zhang2022knowledge,chen2023star,chen2024deep,li2026dyg,lirouting,li2026node}.
Early methods focus on mapping sets to a Hilbert space using permutation-invariant pooling operations~\citep{zaheer2017deep,skianis2020rep,murphy2018janossy,lee2019set,zhang2019fspool,zhanglearning}. 
Some architectures learn optimal permutations directly~\citep{zhanglearning,rezatofighi2018deep} or establish canonical ordering through sorting~\citep{zhang2019fspool}.
Attention-based mechanisms have also become prominent; \citet{lee2019set} adapted the Transformer architecture~\citep{vaswani2017attention} for set data, while \citet{jaegle2021perceiver} utilized cross-attention for permutation invariance. 
Alternative strategies include minimizing discrepancies between sets~\citep{mialon2021trainable}, box embeddings~\citep{lee2022set2box}, and fuzzy representations~\citep{xu2025fuse}. 
For geometric alignments, \citet{skianis2020rep} pioneered Optimal Transport via Bipartite Matching, while \citet{pswe,chenadversarial} introduced Sliced-Wasserstein embedding to capture distributional topology.
Recent advances include meta-learning~\citep{guo2021learning,leeself}, injective embeddings like FSW~\citep{amirfourier}, quantum modeling~\citep{vargas2025quantum}, and protein language applications~\citep{naderializadeh2025aggregating}.
We adopt~\cite{pswe} as our set encoder to leverage Sliced-Wasserstein geometric properties.

\textbf{Distributionally Robust Optimization (DRO).}
DRO improves decision making by optimizing the worst-case expected objective over plausible distributions~\cite{rahimian2019distributionally,rahimian2022frameworks,bertsimas2019adaptive}. 
A central theme of DRO is constructing ambiguity regions~\cite{rahimian2019distributionally,wang2025knowledge,nguyenbeyond,ma2024differentiable,yu2024efficient,inatsu2022bayesian}.
Classical formulations include moment-based and divergence-based approaches such as KL-divergence~\cite{husain2023distributionally,zhou2020planning,namkoong2016stochastic}, while \citet{staib2019distributionally} utilizes kernel mean embeddings in Reproducing Kernel Hilbert Space and \citet{sagawa2019distributionally} minimizes maximum expected loss over predefined groups.
Ambiguity regions defined through Wasserstein distance have gained attention~\cite{kuhn2019wasserstein,kwon2020principled}.
Wasserstein-DRO (WDRO) is appealing due to its geometric interpretability, connections to regularization, and ability to capture distributional shifts~\cite{gao2024wasserstein,micheli2025wasserstein}, with wide applications~\cite{liu2025provable,huang2025effective,wurobust,zhang2025revisiting}.
However, conventional WDRO requires modeling transport costs between individual elements, making direct SRL application computationally expensive.
To address this concern, this work specifically proposes the \model~formulation that is tailored for modeling sets.
By approximating Wasserstein distances through one-dimensional projections, \model~provides an efficient surrogate that preserves distributional robustness while reducing computational cost.

\section{Preliminary and Motivation}

\subsection{Problem Description}
\textbf{Set Representation Learning.}
Let $\mathcal{X}\subseteq\mathbb{R}^d$ be the element space and $\mathcal{Y}$ the label space.
Each input is an unordered set $S=\{x_i\}_{i=1}^{n}$ with variable cardinality $n$, where $x_i\in\mathcal{X}$.
Set representation learning (SRL) learns a permutation-invariant and cardinality-agnostic encoder
$\f:\mathcal{S}(\mathcal{X})\rightarrow\mathbb{R}^c$ that maps $S$ to a vector $v=\texttt{f}(S)$.
The learned set representation $v$ is then used for downstream tasks, such as set classification (predicting label $y\in\mathcal{Y}$) or set matching via embedding-based ranking. We summarize all notations in Appendix~\ref{app:notation}.

A common SRL training objective optimizes the \emph{nominal} (source) risk via empirical risk minimization (ERM)~\cite{zaheer2017deep,lee2019set}.
Let $\mathbb{P}_0$ be the nominal training distribution over clean sets. Conventional ERM solves:
\begin{sequation}
\min_{\texttt{f},\texttt{g}}\ \mathbb{E}_{(S,y)\sim \mathbb{P}_0}\big[\ell\big(\texttt{g}(\texttt{f}(S)),y\big)\big],
\label{eq:erm}
\end{sequation}%
where $\texttt{g}$ denotes the task-specific function (e.g., a classifier or a ranking function), operating on the set representation $v=\texttt{f}(S)$, and $\ell$ the corresponding loss.

\textbf{Set with Inference-Time Element Corruption.}
We study a realistic scenario where training data is carefully curated, but sets observed at inference time suffer from \emph{element-level corruption}.
Given $S$ as the underlying clean set; at inference time, the model however may observe a corrupted version $S'$ in stead of $S$.
Such corruptions are typically \emph{sparse} and \emph{heterogeneous}: only a subset of elements is affected, and corruption patterns vary across instances as well. 
Typical examples include missing elements and spurious outliers.

We assume the corruption is \textit{label-preserving}: it only affects the partial set elements but not the ground-truth label of the holistic set semantic.
Importantly, the learner does not observe which elements are corrupted at inference time; it only receives the full set $S'$.

\subsection{Towards Learning Robust Set Representations}
Minimizing Eq.~\eqref{eq:erm} does not necessarily yield {stable} set representations under element-level corruption.
In SRL, $\f$ is typically implemented by aggregating element-wise features (e.g., pooling or attention)~\cite{zaheer2017deep,lee2019set,zhang2019fspool}, where a small fraction of corrupted elements may disproportionately affect the aggregated representation, even when the corruption is label-preserving.
Consequently, the downstream predictor $\g$, trained on clean embeddings, may fail when evaluated on corrupted embeddings at inference time.

\textbf{Robustness in SRL.}
To fix this gap, a natural robustness principle is to optimize performance under the \emph{worst} plausible corruptions.
Let $\Gamma(S)$ denote the ambiguity region of plausible corruptions for each input $S$ (e.g., capturing sparse and heterogeneous element corruption at inference time).
With $\mathbb{P}_0$ denotes the nominal data distribution over training inputs, one feasible solution is to leverage \textit{Distributionally Robust Optimization (DRO)}~\cite{scarf1957min} as:  
\begin{sequation}
\min_{\f,\g} \mathbb{E}_{(S,y)\sim\mathbb{P}_0} \Big[ \sup_{\mathbb{Q}\in\Gamma(S)} \ \mathbb{E}_{S'\sim \mathbb{Q}}\big[\ell\big(\g(\f(S')),y\big)\big] \Big].
\label{eq:dro}
\end{sequation}%
It can hedge against \emph{unknown} inference-time corruption mechanisms by optimizing the worst-case risk~\cite{rahimian2019distributionally}.
However, for set-structured inputs, specifying $\Gamma(S)$ and solving the inner maximization can be prohibitive: for example, directly enumerating $\Gamma(S)$ over discrete set variants often leads to combinatorial explosion.
These challenges motivate the development of a scalable DRO framework specifically for set representation learning.

\section{Our Method}
\label{sec:method}

\subsection{Overview}
Our approach to learning robust set representations is twofold. 
\ding{172} We formulate sets as empirical measures and view corruptions as distributional perturbations. By characterizing the ambiguity region $\Gamma$ via the Sliced-Wasserstein (SW) distance, we derive the formulation of Sliced-Wasserstein Distributionally Robust Set Optimization (\model). 
\ding{173} To mitigate the high computational cost, we further propose a synthesis-based approximation. This yields a \emph{barycentric adversary} that is essentially a low-dimensional and differentiable parameterization to facilitate the efficient identification of target perturbations.

\subsection{\model~Instantiation for Sets}

\subsubsection{\textbf{Modeling Sets with SW Metric}}
We model a set $S=\{x_i\}_{i=1}^n\subset\mathcal{X}$ as an empirical measure over the element space~\cite{muandet2012learning,zaheer2017deep,edwards2017towards}:
\begin{sequation}
\mu_S = \frac{1}{n}\sum_{i=1}^{n}\delta_{x_i} \in \mathcal{P}(\mathcal{X}),
\label{eq:emp_measure}
\end{sequation}%
where $\delta_x$ is the Dirac measure at $x$ and $\mathcal{P}(\mathcal{X})$ denotes the set of probability measures on $\mathcal{X}$.
This representation is permutation-invariant and explicitly captures the within-set distribution, so element corruption can be naturally viewed as a perturbation from $\mu_S$ to $\mu_{S'}$.

To quantify their discrepancy, we adopt the \emph{2-Sliced-Wasserstein} distance~\cite{kolouri2019generalized}, which is derived from the Optimal Transport (OT) theory~\cite{kantorovich2006translocation}.
For $\mu_S$ and $\mu_{S'}$ over $\mathbb{R}^d$, it is defined as:
\begin{sequation}
\resizebox{0.7\linewidth}{!}{$
\displaystyle
{SW}(\mu_S, \mu_{S'}) = \left(\int_{\mathbb{S}^{d-1}} {W}\!\left(\mu_S^\omega, \mu_{S'}^\omega\right)^2 \mathrm{d}\omega\right)^{\frac{1}{2}},
\label{eq:sliced_wasserstein}
$}
\end{sequation}%
where $\mathbb{S}^{d-1}$ is the unit sphere, and $\mu_{S}^{\omega}$, $\mu_{S'}^{\omega}$ denote the 1D projections under $\omega(x)=\omega^\top x$.
$W$ denotes the standard 2-Wasserstein distance that can computed as:
\begin{sequation}
\resizebox{0.95\linewidth}{!}{$
\displaystyle
\label{eq:was_dis}
{W}(\mu_S^\omega, \mu_{S'}^\omega) = \Big(\inf_{\texttt{p}\in \text{plans}(\mu_S^\omega, \mu_{S'}^\omega)} \int \|x - \texttt{p}(x)\|^{2}\mathrm{d}\mu_S^\omega(x) \Big)^{\frac{1}{2}}.
$}
\end{sequation}%
The infimum is taken over all transport plans between $\mu_S^\omega$ and $\mu_{S'}^\omega$. 
For 1D measures, a closed-form optimal solution $\texttt{p}^+$ exists as the minimizer in Eq.~\eqref{eq:was_dis}.

\begin{figure}[t]
    \begin{center}
    \centerline{\includegraphics[width=\columnwidth]{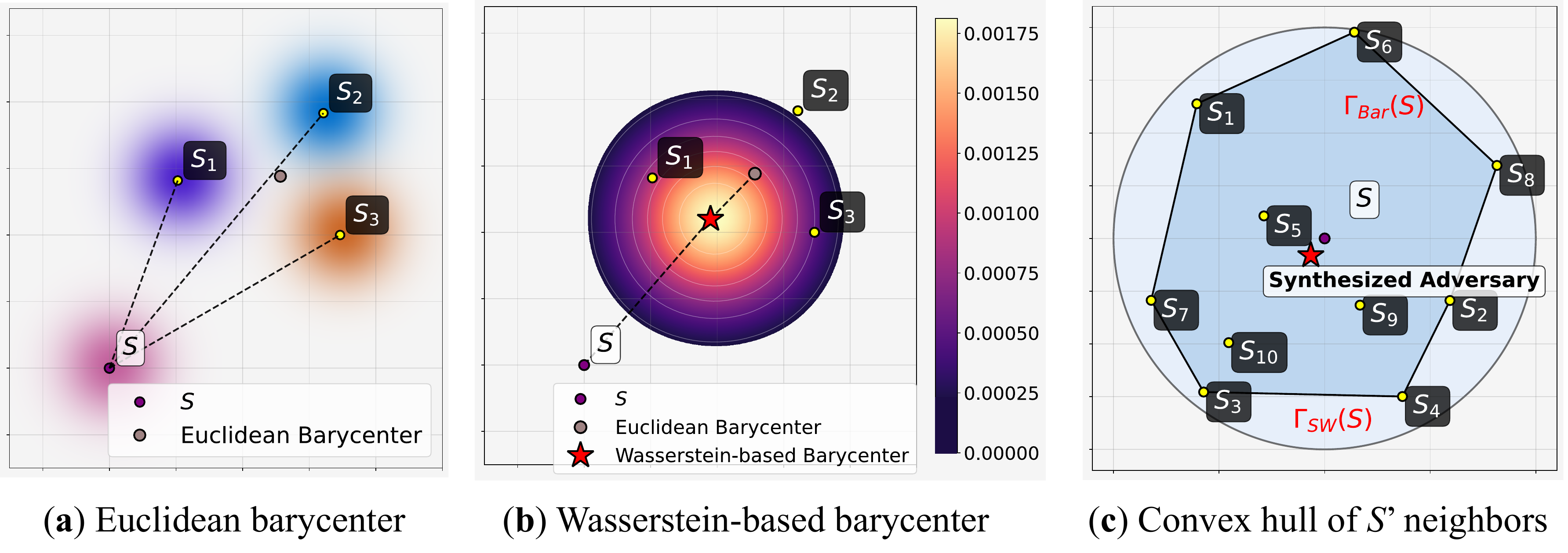}}
    \caption{Illustration of Wasserstein barycentric data synthesis.}
    \label{fig:x}
    \end{center}
\vspace{-0.2cm}
\end{figure}

\subsubsection{\textbf{\model~Instantiation}}
Using SW to define feasible corruptions, we can try to instantiate $\Gamma(S)$ as an SW-ball around $\mu_S$:
\begin{sequation}
\Gamma_{SW}(S) := \left\{S': {SW}(\mu_S, \mu_{S'}) \leq \rho\right\},
\label{eq:uncertainty_set_sw}
\end{sequation}%
where $\rho>0$ controls the corruption radius.
This yields \model~objective as follows:
\begin{sequation}
\min_{\f,\g}
\mathbb{E}_{(S,y)\sim\mathbb{P}_0}\Big[\sup_{S'\in \Gamma_{SW}(S)} \ell\!\left(\g(\f(S')),y\right) \Big].
\label{eq:sw_dro}
\end{sequation}%
A practical instantiation of Eq.~\eqref{eq:sw_dro} may expect a set encoder $\f(\cdot)$ whose representation respects the same SW geometry.
Since standard pooling ignores this property~\cite{skianis2020rep}, we introduce a Wasserstein-aware encoder specifically aligned with \model.

\paragraph{Wasserstein-aware Set Encoder~\cite{pswe}.}
Following~\citet{pswe}, $\f(\cdot)$ can be implemented with the following two stages:
\begin{itemize}[leftmargin=*]
\item It first introduces a learnable reference set $O$ $=$ $\{o_h\}_{h=1}^H$ with empirical measure $\mu_O$ $=$ $\frac{1}{H}\sum_{h=1}^{H}\delta_{o_h}$. For each projection $\omega$, the optimal solution $\texttt{p}^+$ is implemented as:
\begin{sequation}
\texttt{p}^+(t, \mu_S^\omega) = \left(\texttt{F}_{\mu_S^\omega}\right)^{-1}\!\left(\texttt{F}_{\mu_O^\omega}(t)\right), \quad t\in\mathbb{R},
\label{eq:1d_ot_map}
\end{sequation}%
where $\texttt{F}$ and $\texttt{F}^{-1}$ are the CDF and quantile function.
Let $T_O^\omega$ $=$ $\texttt{sort}(\{\omega^\top o_h\}_{h=1}^H)$ and $T_S^\omega$ $=$ $\texttt{sort}(\{\omega^\top x_i\}_{i=1}^n)$.
For empirical measures, $\texttt{p}^+$ reduces to quantile matching between sorted scalars; 
Let $t_h^\omega$ be the $h$-th entry of $T_O^\omega$. When $H$ $\neq$ $n$, a simple index-based quantile approximation can be implemented:
\begin{sequation}
\texttt{p}^+(t_h^\omega,\mu_S^\omega) := T_S^\omega[q_h], \quad q_h := \Big\lceil \frac{n}{H}h\Big\rceil.
\label{eq:rank_matching}
\end{sequation}%

\item It then introduces $R$ Monte Carlo projections $\Omega$ $=$ $\{\omega_r\}_{r=1}^R$ to approximate the integral in Eq.~(\ref{eq:sliced_wasserstein}).
Therefore, we aggregate all OT coordinates by concatenation:
\begin{sequation}
\resizebox{0.95\linewidth}{!}{$
\displaystyle
\f_{\mathrm{SW}}(\mu_S) := \frac{1}{\sqrt{RH}}\texttt{Concat}_{r=1, h=1}^{R,H} \Big(\texttt{p}^+(t_h^{\omega_r}, \mu_S^{\omega_r})\Big) \in \mathbb{R}^{RH}.
\label{eq:wa_encoder}
$}
\end{sequation}%
\end{itemize}

\subsection{Data\,Synthesis\,for\,\model\,Approximation}
\label{subsec:data_synthesis_sw_dro}
Directly optimizing the inner supremum in Eq.~\eqref{eq:sw_dro} over the SW-ball is still intractable, as it requires searching over corrupted sets $S'$ that both satisfy the Sliced-Wasserstein constraint and maximize the downstream task loss after passing through $g$.
Our goal is therefore to \emph{approximately characterize} the most adversarial directions within an SW neighborhood while remaining consistent with \ding{172} the SW geometry that defines the corruption region and \ding{173} the task loss that determines worst-case behavior.
To this end, as shown in Figure~\ref{fig:x}, we introduce a \emph{barycentric synthesis} surrogate, which provides a \emph{low-dimensional and differentiable} parameterization of SW-consistent perturbations, enabling efficient search for high-quality adversarial candidates.
In practice, data synthesis has been used broadly to improve representation learning and model adaptation in many fields~\cite{he2026pedagogically,he2026preserving,he2025paser}.

\subsubsection{\textbf{Wasserstein Barycenter Synthesis}}
\label{sec:wbs}
Let $v_S$ denote the set embedding from $\f_{\mathrm{SW}}(\mu_S)\in\mathbb{R}^{RH}$.
For each $S$ in a minibatch $\mathcal{B}$, we build a local neighbor pool $\mathcal{G}(S)=\{S_k\}_{k=1}^K\subseteq\mathcal{B}$ by $K$-NN under $\ell_2$ set embedding distance\footnote{In practice, we select either the $K$ nearest neighbors or all samples within Eculidean distance $\rho$, whichever is smaller. For notational simplicity, we denote the resulting set size as $K$.}.
We denote the probability simplex by:
\begin{sequation}
\Delta_K := \big\{\Lambda\in\mathbb{R}^K:\ \sum_{k=1}^K\lambda_k=1 \text{ and } \lambda_k \geq 0 \big\}.
\end{sequation}%
Given $\Lambda=(\lambda_1,\ldots,\lambda_K)\in\Delta_K$, we synthesize a perturbed embedding by convex mixing as follows:
\begin{sequation}
\bar v_S(\Lambda) := \sum_{k=1}^K \lambda_k\, v_{S_k}.
\label{eq:barycenter_embedding}
\end{sequation}%

Despite simplicity, Eq.~\eqref{eq:barycenter_embedding} remains the set semantics within the SW-ball, as shown in Prop.~\ref{prop:jensen_locality} (all theoretical proofs are deferred to Appendix~\ref{app:proof}).
\begin{proposition}
\label{prop:jensen_locality}
For any $\Lambda\in\Delta_K$, let $S_\Lambda$ denote the implicit (virtual) set that corresponds to the synthesized embedding $\bar v_S(\Lambda)$.
We have $S_\Lambda \in \Gamma(S)$.
\end{proposition}
Moreover, this synthesis is \emph{not} an arbitrary interpolation in representation space.
Instead, it has a precise optimal-transport interpretation: under the 1D Wasserstein geometry, $\bar v_S(\Lambda)$ corresponds to the encoding of \textbf{Wasserstein barycenters}, i.e., the Fr\'echet mean that minimizes a weighted sum of squared $W_2$ distances on each projection.
\begin{proposition}
\label{prop:barycenter_semantics}
Given reference measure $\mu_O$, for any batch pool $\mathcal{G}_S$ and weights $\Lambda\in\Delta_K$, the synthesized embedding $\bar v_S(\Lambda)$ corresponds to a (virtual) barycentric set, i.e., the SW Fr\'echet mean of all sets in $\mathcal{G}_S$:
\begin{sequation}
\resizebox{0.95\linewidth}{!}{$
\displaystyle
\bar v_S(\Lambda) = \f_{\mathrm{SW}}(\mu_{\Lambda}), \mu_{\Lambda} \in \arg\min_{\mu\in\mathcal{P}(\mathcal{X})} \sum_{k=1}^K \lambda_k\,{SW}^2(\mu,\mu_{S_k}).
\label{eq:barycenter_min_simple}
$}
\end{sequation}%
\end{proposition}
$\bar v_S(\Lambda)$ represents a most representative point within the local latent space, where the aggregate distance to all neighboring sets is minimized.
This property enables us to restrict the adversary to barycentric perturbations without leaving the intended SW region, leading to a tractable surrogate of the original inner maximization as follows.

\subsubsection{\textbf{Barycentric Adversary As a Surrogate}}
Based on the introduction of $\Delta_K$, we then reformulate the corruption region as the convex hull of embeddings:
\begin{sequation}
{\Gamma}_{Bar}(S):= \big\{\bar v_S(\Lambda):\Lambda\in\Delta_K\big\}.
\label{eq:structured_uncertainty}
\end{sequation}%
The resulting surrogate replaces the inner supremum with a low-dimensional maximization over $\Delta_K$:
\begin{sequation}
\min_{\f_{\mathrm{SW}},\g}
\mathbb{E}_{(S,y)\sim\mathbb{P}_0} \Big[ \max_{\Lambda\in\Delta_K}\; \ell\big(\g(\bar v_S(\Lambda)),y\big) \Big].
\label{eq:bar_dro_objective}
\end{sequation}%
With this construction, the barycentric inner maximization in Eq.~\eqref{eq:bar_dro_objective} upper-bounds the discrete inner objective over individual sets. As a result, the adversary is strictly more expressive, yielding a more conservative surrogate that better captures worst-case behavior. This relationship is formalized in the following proposition:
\begin{proposition}
\label{prop:disc_le_bar}
Let $\mathcal{L}_{{Disc}}(S)$ $:=$ $\max_{k\in[K]} \ell\big(\g(v_{S_k}),y\big)$ and $\mathcal{L}_{{Bar}}(S)$ $:=$ $\max_{\Lambda\in\Delta_K} \ell\big(\g(\bar v_S(\Lambda)),y\big)$.
For any predictor $\g$ and loss $\ell$, we have:
\begin{sequation}
\mathcal{L}_{{Disc}}(S)\ \le\ \mathcal{L}_{{Bar}}(S).
\label{eq:disc_le_bar}
\end{sequation}%
\end{proposition}
Intuitively, as shown in Figure~\ref{fig:x}(c), $\Gamma_{\mathrm{Bar}}(S)$ enlarges the feasible adversarial set from a finite vertex set $\{v_{S_k}\}$ to its convex hull, hence the resulting inner maximization provides an upper bound of the discrete counterpart.
We defer to Appendix~\ref{app:lipschi} a quantitative analysis, which characterizes the tightness of the barycentric surrogate by bounding $L_{\mathrm{Bar}}(S)-L_{\mathrm{Disc}}(S)$ under a local Lipschitz condition.

So far, we show that our barycentric synthesis \ding{172} \emph{continuizes} the adversary beyond discrete samples while remaining major semantics within the local set neighborhood, and \ding{173} turns the intractable set-space search into an efficient low-dimensional maximization over $\Delta_K$.

\subsection{Overall Objective and Optimization}
In practice, we optimize a weighted combination of the nominal ERM loss and the barycentric robust term.
Given $\alpha \ge 0$ as a trade-off hyperparameter, our overall training objective is formulated as:
\begin{sequation}
\min_{\f_{\mathrm{SW}},\,\g}\ 
\mathbb{E}_{(S,y)\sim\mathbb{P}_0}
\Big[
\ell\big(\g(v_S),y\big)
+
\alpha\cdot
\max_{\Lambda\in\Delta_K}\ \ell\big(\g(\bar v_S(\Lambda)),y\big)
\Big],
\label{eq:final_training_objective}
\end{sequation}%

For each training set $S$, we approximately solve the inner maximization in Eq.~\eqref{eq:final_training_objective} using projected gradient ascent~\cite{levitin1966constrained} on $\Delta_K$.
We initialize the mixing weights with the uniform distribution
$\Lambda^{(1)}=\left(\frac{1}{K},\ldots,\frac{1}{K}\right)$ and iteratively perform the following updates:
\begin{align}
\tilde{\Lambda}^{(t+1)} &= \Lambda^{(t)} + \eta\,\nabla_{\Lambda}\ell\big(\g(\bar v_S(\Lambda^{(t)})),y\big),\\
\Lambda^{(t+1)} &= \Pi_{\Delta_K}\!\left(\tilde{\Lambda}^{(t+1)}\right),
\end{align}
where $\eta>0$ is the ascent step size and $\Pi_{\Delta_K}$ denotes the Euclidean projection onto $\Delta_K$
(enforcing $\lambda_k\ge 0$ and $\sum_{k=1}^K \lambda_k=1$).
After $T$ ascent steps, we obtain $\Lambda^\star=\Lambda^{(T)}$.
We then fix $\Lambda^\star$ (by stopping gradients through $\Lambda^\star$) and update $\f_{\mathrm{SW}}$ and $\g$ by standard gradient descent on Eq.~\eqref{eq:final_training_objective}.
We detail the pseudo-codes of our \model~in Appendix~\ref{app:algo} and time complexity analysis in Appendix~\ref{app:complexity}.

\section{Experiments}

\subsection{Experiment Setups}
\paragraph{Task Descriptions.}
We validate the robustness of our method across four downstream tasks subject to varied inference-time corruptions.
By encompassing domains ranging from 3D geometry to visual understanding, we expect a systematic evaluation under diverse data modalities.
\begin{itemize}[leftmargin=*]
\item \textbf{Similar Set Ranking (Task I)}: At inference time, the \emph{query set} is corrupted by noise (e.g., injected irrelevant elements and/or missing true elements). The goal is to retrieve and rank candidate sets that remain most similar to the underlying clean query, based on embedding distances. This task evaluates robustness of \emph{set-to-set similarity} modeling under noisy query composition.

\item \textbf{Point Cloud Classification (Task II)}: The task classifies 3D objects represented as unordered point sets by aggregating local geometric information into a global representation.
At inference time, the point cloud is corrupted by \emph{inaccurate observations} (e.g., perturbed coordinates), while the input remains an unordered set of points. 
The model evaluates whether geometric aggregation is resilient to corrupted data.

\item \textbf{Topic Set Expansion (Task III)}: The task expands a topic by selecting relevant keywords from a vocabulary given a seed keyword set, assessing semantic coherence modeling at the set level. At inference time, the \emph{seed keyword set} is imperfect, containing missing topic-defining words and/or noisy, off-topic words. 
We assess robustness of semantic set modeling under such conditions.

\item \textbf{Patch-set Visual Recognition (Task IV)}: The task aims to perform visual classification by encoding sets of local patch elements. At inference time, element corruption is introduced at the \emph{patch level}, where a subset of patches is masked out or perturbed by additive noise. The model is therefore required to perform robust visual recognition against localized corruptions.
\end{itemize}

\paragraph{Data Descriptions.}
We adopted a standardized data partitioning protocol across all tasks: datasets were split into training and testing subsets with an 80:20 ratio. The training portion was subdivided (80:20) to create a validation set for hyperparameter tuning. 
To evaluate robustness under inference-time corruption, datasets are grouped with different corruption levels. 
Specifically, only the validation and test sets are partitioned into \textit{clean/mild/severe} subsets. 
Mild and severe parts introduce corruption to approximately 10\% and 40\% of the input elements or regions.
This protocol trains models on clean data but evaluates them under varying degrees of corrupted inputs, reflecting real deployment conditions.
Data statistics are reported in Appendix~\ref{app:data}. 

\paragraph{Baselines.} 
We compare against a diverse suite of set representation baselines, including classic pooling (\textbf{MeanP/MaxP})~\citep{lin2013network}, universal set-function modeling (\textbf{DeepSet})~\citep{zaheer2017deep}, prototype matching via OT theory (\textbf{RepSet})~\citep{skianis2020rep}, attention-based set Transformers (\textbf{SetTRSM})~\citep{lee2019set}, information-theoretic learning (\textbf{DIEM})~\citep{kim2022differentiable}, learnable sorting-based pooling (\textbf{FSPool})~\citep{zhang2019fspool}, and recent strong OT-based methods such as \textbf{PSWE}~\citep{pswe} and frequency-domain multiset modeling (\textbf{FSW})~\citep{amirfourier}. 
Detailed descriptions of these baselines are provided in the Appendix~\ref{app:baselines}.

\paragraph{Configurations.}
Optimization for all other models is performed using the default Adam optimizer~\citep{adam}, with the learning rate tuned via grid search within the set $\{10^{-4}, 10^{-3}, 10^{-2}\}$. 
Regarding baseline configurations, we strictly followed the hyperparameter settings reported in their original literature; for models without specified configurations, optimal settings were determined via grid search. The implementation was carried out in PyTorch 1.13.0 (Python 3.9) on a Linux server equipped with eight NVIDIA L40S GPUs. To account for statistical variance, all reported results represent the mean performance and standard deviations over five independent runs. 
Comprehensive details for reproducibility are provided in Appendix~\ref{app:hyperparams}.

\subsection{Experimental Results}

\subsubsection{\textbf{Task I Evaluation}}
For the Similar Set Ranking task, we evaluate on two large-scale real-world social network datasets, \textit{Friendster}~\citep{DBLP:journals/kais/YangL15} and \textit{LIVEJ}~\citep{DBLP:conf/imc/MisloveMGDB07}.
Figure~\ref{fig:fs_livej_band} summarizes the results under different corruption conditions.
Sets are ranked by measuring the Euclidean distance in the learned embedding space, with Recall@k and NDCG adopted as the evaluation metrics, which are widely used in retrieval and recommendation settings~\cite{he2024interpretable,chen2023sim2rec,luo2025recranker,luo2026integrating,chen2022modeling,chen2022attentive,cui2024context,cui2026semantic,zhang2026token}.

% ---------- In document ----------
\begin{figure}[t] 
  \centering

  % ---------- Row 1: FS Dataset ----------
  \begin{subfigure}{0.48\linewidth}
    \centering
    % 对应文件名: fs-overall-recall-band.pdf
    \includegraphics[width=\linewidth]{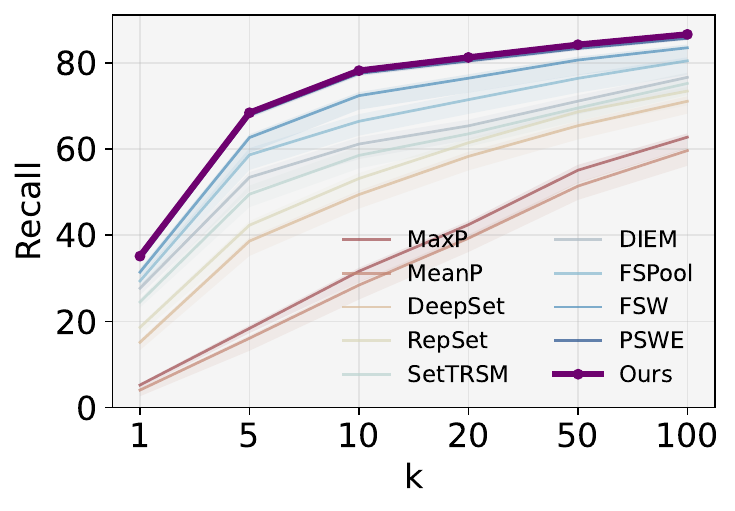}
    \caption{Friendster-Overall-Recall@k}
    \label{fig:fs-recall}
  \end{subfigure}%
  \begin{subfigure}{0.48\linewidth}
    \centering
    % 对应文件名: fs-overall-ndcg-band.pdf
    \includegraphics[width=\linewidth]{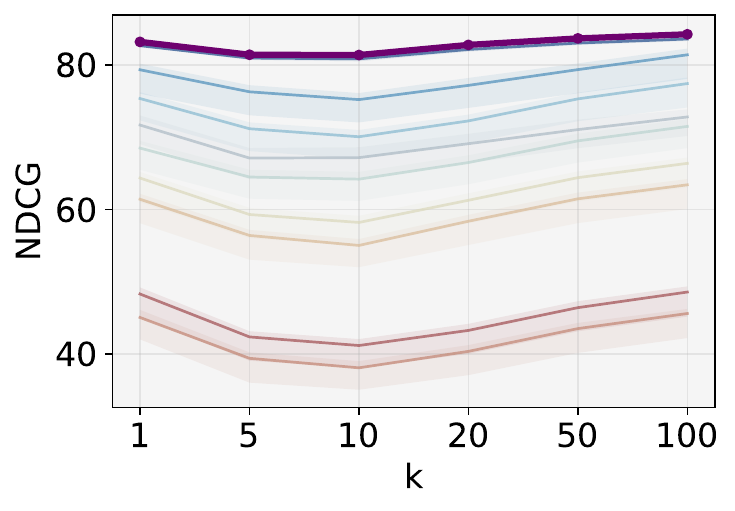}
    \caption{Friendster-Overall-NDCG@k}
    \label{fig:fs-ndcg}
  \end{subfigure}

  % ---------- Row 2: LiveJ Dataset ----------
  \begin{subfigure}{0.48\linewidth}
    \centering
    % 对应文件名: livej-overall-recall-band.pdf
    \includegraphics[width=\linewidth]{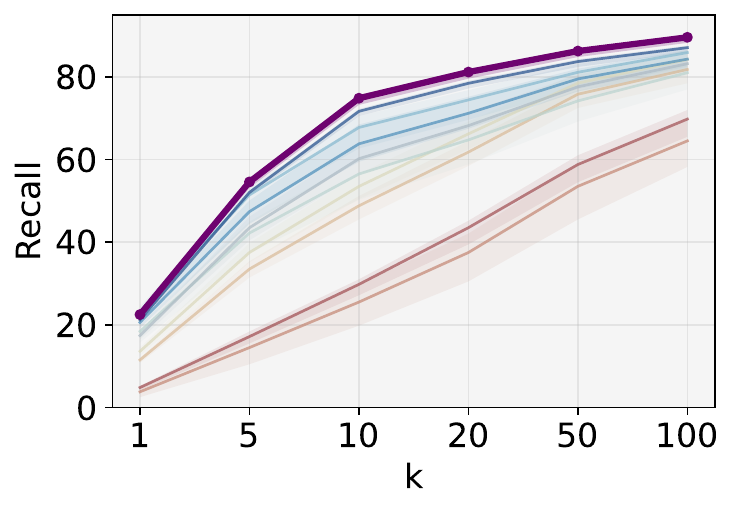}
    \caption{LiveJ-Overall-Recall@k}
    \label{fig:livej-recall}
  \end{subfigure}%
  \begin{subfigure}{0.48\linewidth}
    \centering
    % 对应文件名: livej-overall-ndcg-band.pdf
    \includegraphics[width=\linewidth]{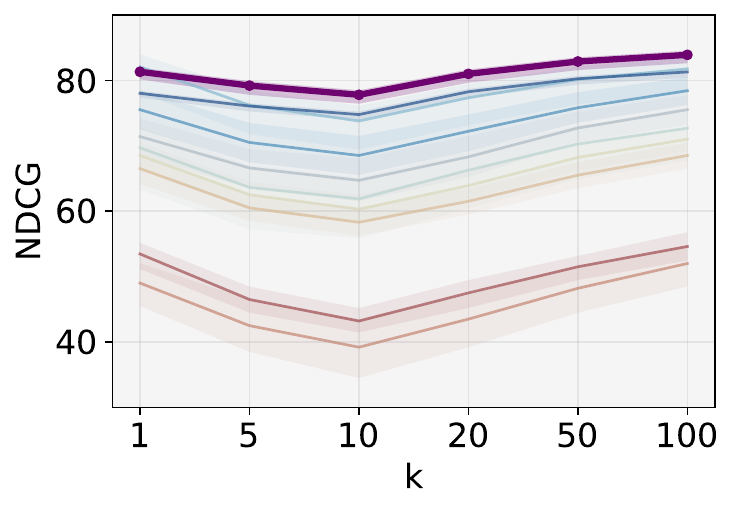}
    \caption{LiveJ-Overall-NDCG@k}
    \label{fig:livej-ndcg}
  \end{subfigure}

  \caption{Task I Performance comparison in terms of Recall@k and NDCG@k. Solid curves correspond to the overall performance and the shaded region indicates the range bounded by the upper clean split and the lower severe split.}
  \label{fig:fs_livej_band}
\end{figure}

We have two major observations:
\ding{172} First, our method in purple consistently achieves strong performance across all datasets and metrics, demonstrating clear overall effectiveness. 
\ding{173} Second, our method exhibits the narrowest shaded band among all compared methods, meaning the performance gap between the severe and clean splits is the smallest. This indicates that our approach is more robust to the inference-time corruption, maintaining stable ranking quality even under challenging conditions. 
Complete numerical results are provided in Appendix~\ref{app:task1}.

\begin{table}[t]
\caption{Task II performance comparison (\%).
Best and second-best results are shown in \textbf{bold} and \underline{underlined}.}
\label{tab:task_pointcloud}
\centering
\small
\begin{minipage}{\linewidth}
\centering
\resizebox{\linewidth}{!}{%
\begin{tabular}{l|cccc}
\toprule
Method & Overall & Clean & Mild & Severe \\
\midrule
MeanP          & 80.41$\pm$0.79 & 84.85$\pm$0.73 & 83.96$\pm$1.18 & 63.97$\pm$3.03 \\
MaxP           & 77.29$\pm$0.73 & 85.49$\pm$0.70 & 79.10$\pm$0.98 & 54.08$\pm$2.82 \\
DeePSet         & 78.67$\pm$0.74 & 85.80$\pm$0.67 & 80.83$\pm$0.99 & 57.60$\pm$2.92 \\
RepSet         & 79.55$\pm$0.76 & 85.09$\pm$0.69 & 82.93$\pm$1.06 & 60.63$\pm$2.97 \\
SetTRSM 	   & 75.28$\pm$0.65 & \underline{86.17$\pm$0.48} & 75.43$\pm$0.85 & 47.81$\pm$2.76 \\
DIEM           & 80.59$\pm$0.68 & 85.10$\pm$0.90 & 83.80$\pm$0.95 & 64.50$\pm$2.10 \\
FSPool         & 80.16$\pm$0.83 & 85.24$\pm$1.41 & 83.51$\pm$0.78 & 62.43$\pm$1.83 \\
FSW            & 81.14$\pm$0.74 & 83.84$\pm$0.59 & \underline{84.62$\pm$1.12} & 69.19$\pm$2.93 \\
PSWE           & \underline{81.50$\pm$0.70} & 84.51$\pm$0.49 & 84.38$\pm$1.08 & \underline{69.65$\pm$2.83} \\
\midrule
\textbf{Ours}         & \textbf{82.86$^*$ $\pm$0.71} & \textbf{86.32$\pm$0.78} & \textbf{85.68$^*$ $\pm$0.73} & \textbf{69.99$^*$ $\pm$2.78} \\
\bottomrule
\end{tabular}
}
\end{minipage}
\end{table}

\subsubsection{\textbf{Task II Evaluation}}
We evaluate point cloud classification on the standard ModelNet benchmark~\cite{wu20153d} with ISAB as the feature backbone~\cite{lee2019set}, with accuracy results summarized in Table~\ref{tab:task_pointcloud}. 
\ding{172} Our method achieves the best performance across all evaluation settings over the strongest baseline in terms of overall accuracy. This demonstrates the effectiveness of our approach for learning robust set representations in point cloud data classification.
\ding{173} When examining different corruption levels, we observe that existing methods exhibit varying robustness performance. 
In particular, PSWE shows relatively strong resistance under severe corruption, achieving the best baseline performance in this regime, while FSW performs competitively under mild corruption, ranking as the second-best baseline. 
These results suggest that different aggregation mechanisms emphasize different aspects of robustness. Nevertheless, our method consistently outperforms all baselines under clean, mild, and severe conditions, indicating improved stability across a wide range of inference-time corruptions.
\ding{174} To assess the statistical reliability of these improvements, we conduct the Wilcoxon signed-rank significance test across all experimental settings. The results show that the performance gains of our method are statistically significant at the 95\% confidence level for the vast majority of comparisons.

\begin{figure}[t]
    \begin{center}
    \centerline{\includegraphics[width=\columnwidth]{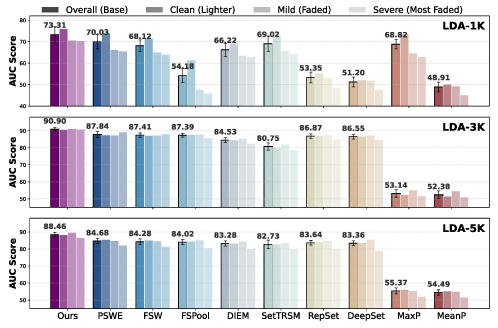}}
    \caption{Task III performance comparison. For each method, the leftmost black-bordered bar represents the overall score, with exact values annotated on top. The symbol ``\texttt{I}'' measures the difference between the clean and severe performances.}
    \label{fig:task3}
    \end{center}
\vspace{-0.5cm}
\end{figure}

\subsubsection{\textbf{Task III Evaluation}}
We conduct evaluations on three benchmark datasets, \textit{LDA-1k}, \textit{LDA-3k}, and \textit{LDA-5k}~\citep{zaheer2017deep}, which differ in both scale and topical coverage.
Tokens are embedded via word2vec~\citep{mikolov2013efficient} and scored by semantic similarity to the query, with ranking quality measured by AUC.
As shown in Figure~\ref{fig:task3}, \ding{172} our method consistently achieves the best overall AUC performance across all three datasets. 
Notably, this advantage is driven by strong performance under the clean data partitions, indicating that our approach effectively captures the core semantic structure required for topic expansion.
\ding{173} Beyond overall accuracy, our method also demonstrates superior robustness to corruption.
Across all datasets, the performance gap between the clean and severe conditions is relatively small compared to competing methods, as indicated by the I-bar markers. This effect is particularly pronounced on \textit{LDA-3k}, where our method exhibits the smallest clean--severe discrepancy, suggesting strong resilience to inference-time corruption.
On \textit{LDA-1k} and \textit{LDA-5k}, although the gap slightly increases, our approach remains highly competitive and consistent performance.
Detailed results are reported in Appendix~\ref{app:task3}.

\begin{table}[t]
\caption{Task IV performance comparison (\%).}
\label{tab:task_accuracy}
\centering
\small
\begin{minipage}{\linewidth}
\centering
\resizebox{\linewidth}{!}{%
\begin{tabular}{l|cccc}
\toprule
Method & Overall & Clean & Mild & Severe \\
\midrule
MaxP    & 48.27$\pm$0.55 & 55.43$\pm$0.67 & 50.18$\pm$0.84 & 27.53$\pm$1.08 \\
MeanP   & 49.99$\pm$0.58 & 56.83$\pm$0.42 & 52.28$\pm$0.64 & 29.46$\pm$0.83 \\
DeepSet & 48.49$\pm$0.85 & 56.12$\pm$0.65 & 51.29$\pm$0.91 & 25.21$\pm$1.18 \\
RepSet  & 49.52$\pm$0.59 & 57.19$\pm$0.61 & 51.82$\pm$0.78 & 26.88$\pm$0.96 \\
SetTRSM & 50.26$\pm$0.71 & 59.48$\pm$0.54 & 53.04$\pm$0.73 & 23.06$\pm$1.24 \\
DIEM    & 51.91$\pm$0.48 & 59.42$\pm$0.71 & 53.67$\pm$0.85 & 30.51$\pm$1.09 \\
FSPool  & \textbf{53.13$\pm$0.49} & \textbf{61.88$\pm$0.69} & 54.15$\pm$0.92 & 29.74$\pm$1.31 \\
FSW     & 52.15$\pm$0.65 & 59.85$\pm$0.55 & 53.45$\pm$0.78 & 30.95$\pm$0.95 \\
PSWE    & 51.21$\pm$0.42 & 57.76$\pm$0.27 & 53.49$\pm$0.53 & \underline{31.43$\pm$0.63} \\
\midrule
\textbf{Ours}  & \underline{53.03$\pm$0.89} & \underline{60.33$\pm$1.05} & \textbf{54.95$^*$ $\pm$1.19} & \textbf{31.89$^*$ $\pm$0.79} \\
\bottomrule
\end{tabular}
}
\end{minipage}
\end{table}

\subsubsection{\textbf{Task IV Evaluation}}
For Task IV, we evaluate on the large-scale NWPU-RESISC45 dataset~\citep{cheng2017remote}. 
We use the same MLP-based patch encoder in~\cite{pswe} and report classification accuracy as the evaluation metric.
\ding{172} In general, our model attains highly competitive overall accuracy and is effectively on par with the best-performing baseline.
\ding{173} Although our method ranks second on the clean subset, it outperforms most competitors on the mild and severe subsets. This indicates improved robustness: the model maintains higher accuracy under more challenging conditions, showing that the proposed method achieves state-of-the-art accuracy under varying test conditions.

\subsection{Empirical Analysis}
We evaluate our framework's design choices using the LiveJ dataset (introduced in Task 1).

\begin{table}[t]
    \centering
    \caption{Ablation study on LiveJ dataset. }
    \label{tab:ablation_study}
    \resizebox{\linewidth}{!}{
        \begin{tabular}{lcccccccc}
            \toprule
            \multirow{2}{*}{\texttt{W/o}} & \multicolumn{4}{c}{\textbf{Recall@10}} & \multicolumn{4}{c}{\textbf{NDCG@10}} \\
            \cmidrule(lr){2-5} \cmidrule(lr){6-9}
             & \textbf{Overall} & \textbf{Clean} & \textbf{Mild} & \textbf{Severe} & \textbf{Overall} & \textbf{Clean} & \textbf{Mild} & \textbf{Severe} \\
            \midrule
            \texttt{\model} 	& 72.58 & 73.45 & 72.22 & 70.96 & 74.40 & 75.23 & 74.14 & 72.69 \\
            \texttt{$f_{sw}$} 	& 69.55 & 70.23 & 69.17 & 68.43 & 70.77 & 71.38 & 70.77 & 69.23 \\
            \texttt{WBS}    	& 72.92 & 73.81 & 72.63 & 71.13 & 75.68 & 76.24 & 75.40 & 74.72 \\
            \textbf{Ours}       & \textbf{74.86} & \textbf{75.64} & \textbf{74.63} & \textbf{73.25} & \textbf{77.80} & \textbf{78.49} & \textbf{77.54} & \textbf{76.44} \\
            \bottomrule
        \end{tabular}}
\end{table}

\subsubsection{\textbf{Module-wise Ablation Study}}
We construct three ablation variants against the full model. The comparison results are presented in Table~\ref{tab:ablation_study}.
\begin{itemize}[leftmargin=*]
\item \texttt{w/o \model}: We eliminate our proposed \model~framework entirely. Instead, the model is directly optimized using standard Empirical Risk Minimization in Eq.~\eqref{eq:final_training_objective}. We notice that, the performance drop, particularly in the \textit{severe} corruption setting, indicates the critical role of our \model~frame in enhancing robustness.

\item \texttt{w/o $f_{sw}$}: We replace the proposed Sliced-Wasserstein-based set encoding mechanism with a rudimentary mean pooling strategy to assess the efficacy of the encoding scheme.  
This variant exhibits the most substantial degradation across all metrics. This result empirically verifies that the Wasserstein-based geometric awareness can capture complex set structures, which is more effective than simple aggregation strategy.

\item \texttt{w/o WBS}: We disable the Wasserstein Barycenter Synthesis module. Instead of synthesizing barycenters, we utilize set embeddings randomly sampled from the current batch as the optimization targets. 
The decline in performance suggests that our synthesized barycenters can effectively provide supervision signals compared to random in-batch samples.
\end{itemize}

\begin{table}[t]
    \centering
    \caption{Empirical study of our approximation strategies. Time represents the training time per epoch.}
    \label{tab:variants}
    \resizebox{0.95\linewidth}{!}{
        \begin{tabular}{p{2.5cm}ccccc}
            \toprule
            \textbf{Method} & \textbf{Overall} & \textbf{Clean} & \textbf{Mild} & \textbf{Severe} & \textbf{Time (min)} \\
            \midrule
            \texttt{Sinkhorn} & 74.81 & 75.52 & \textbf{74.71} & 73.20 & 21.5 \\
            \texttt{Disc-Adversary} & 74.45 & 75.18 & 74.01 & 73.27 & 18.4 \\
            \midrule
            \textbf{Ours}       & \textbf{74.86} & \textbf{75.64} & 74.63 & \textbf{73.25}  & \textbf{15.6} \\
           	\bottomrule
        \end{tabular}
    }
\end{table}

\subsubsection{\textbf{Study of Our Approximation Strategies}}
Our framework employs a twofold approximation strategies: We first adopt Sliced-Wasserstein metric for tractable Wasserstein geometry, and then we propose \textit{barycentric adversary} as a differentiable surrogate for the inner maximization in DRO. 
We thus conduct empirical studies to validate the effectiveness of these two designs.

\paragraph{Wasserstein Geometry Approximation.}
To evaluate the efficacy of our approximation, we substitute the SW metric with the standard Wasserstein metric. Given that the computational complexity of the original Wasserstein distance scales at $O(n^3 \log n)$ or higher~\cite{pele2009fast}, we adopt the Sinkhorn computation algorithm as a common practice~\cite{cuturi2013sinkhorn}. Consequently, we implement a Sinkhorn-based variant for both set encoding and the \model~formulation. 
As shown in Table~\ref{tab:variants}, while the \texttt{Sinkhorn} baseline achieves competitive overall performance, our SW-based method still outperforms it with significantly improved computational efficiency, i.e., 21.5 min vs. 15.6 min per epoch (27\% faster). This demonstrates that SW approximation provides a better trade-off between accuracy and efficiency for large-scale settings.

\paragraph{Barycentric Adversary for Optimization.}
The second approximation lies within the optimization process, where we employ a barycentric adversary as a surrogate. We compare our approach against the \textit{Discrete Adversary} strategy (corresponding to $\mathcal{L}_{Disc}$ in Proposition~\ref{prop:disc_le_bar}), which searches over all discrete sets in $\mathcal{G}$ to identify the worst-case instance for gradient estimation. 
As shown in Table~\ref{tab:variants}, our method outperforms \texttt{Disc-Adversary}, empirically validating Proposition~\ref{prop:disc_le_bar}, while achieving better training efficiency.

% \begin{figure}[t]
%     \begin{center}
%     \centerline{\includegraphics[width=\columnwidth]{figs/loss_comparison.pdf}}
%     \caption{Loss curves of ERM vs \model.}
%     \label{fig:loss}
%     \end{center}
% \end{figure}

\paragraph{Comparison with Random Combinatorial Search.}
To further examine whether an explicit search over corrupted set candidates can replace the barycentric adversary, we evaluate random combinatorial sampling (RCS) on the LiveJ dataset. RCS samples candidate corrupted sets in the discrete set space for $m$ rounds and uses high-loss candidates for robust training. We also evaluate an SW-constrained variant that filters sampled candidates by the local Sliced-Wasserstein distance. As shown in Table~\ref{tab:rcs}, increasing the number of sampled candidates substantially increases training time, while the performance gain remains limited. In contrast, our barycentric adversary does not require a sampling-round parameter and achieves 74.86 Recall@10 with 15.6 min per epoch, yielding a better accuracy-efficiency trade-off by optimizing continuously over simplex weights rather than sampling discrete corrupted sets.

\begin{table}[t]
\centering
\caption{Random combinatorial search analysis on LiveJ. R@10 denotes Recall@10, $m$ denotes the number of sampling rounds, and time is measured per epoch.}
\label{tab:rcs}
\resizebox{0.95\linewidth}{!}{
\begin{tabular}{p{2.5cm}cccccc}
\toprule
\multirow{2}{*}{\textbf{Method}} & \multicolumn{2}{c}{$\boldsymbol{m=2}$} & \multicolumn{2}{c}{$\boldsymbol{m=5}$} & \multicolumn{2}{c}{$\boldsymbol{m=10}$} \\
\cmidrule(lr){2-3} \cmidrule(lr){4-5} \cmidrule(lr){6-7}
 & \textbf{R@10} & \textbf{Time} & \textbf{R@10} & \textbf{Time} & \textbf{R@10} & \textbf{Time} \\
\midrule
\texttt{RCS}
& 69.45 & 21.7
& 70.56 & 23.5
& 71.13 & 27.4 \\
\texttt{RCS w/ SW}
& 69.52 & 22.6
& 71.81 & 24.8
& 71.74 & 28.1 \\
\bottomrule
\end{tabular}
}
\end{table}

% \paragraph{Optimization Loss Curves.}
% We zoom into the first 10 epochs to compare the loss curves of ERM and ours.
% As shown in Figure~\ref{fig:loss}, compared with ERM, \model~shows a consistently higher training loss, which is expected since it optimizes a more conservative worst-case objective.
% Despite the higher loss scale, \model~converges steadily with a clear downward trend across steps/epochs. Occasional stage-wise spikes appear, but they are transient and do not affect overall stability or convergence.

\begin{table}[t]
    \centering
    \caption{Comparison with other DRO variants.}
    \label{tab:dro_variants}
    \resizebox{0.95\linewidth}{!}{
        \begin{tabular}{p{2.5cm}ccccc}
            \toprule
            \textbf{Method} & \textbf{Overall} & \textbf{Clean} & \textbf{Mild} & \textbf{Severe} & \textbf{Time (min)} \\
            \midrule
            \texttt{WDRO} & 74.72 & 75.32 & 74.21 & 73.10 & 17.5 \\
            \texttt{KL-DRO}    & 74.45 & 74.92 & 74.14 & 73.87 & \textbf{14.2} \\
            \texttt{MMD-DRO} & 74.09 & 74.86 & 73.60 & 73.07 & 18.7 \\
            \midrule
            \textbf{Ours}       & \textbf{74.86} & \textbf{75.64} & \textbf{74.63} & \textbf{73.25} & 15.6 \\
           	\bottomrule
        \end{tabular}
    }
\end{table}

\paragraph{Cardinality-wise Scalability.}
We further evaluate scalability by grouping LiveJ data instances according
to the input set cardinality, i.e., the number of elements contained in
each input set. Each cardinality group corresponds to a range of set
sizes, such as sets with 1--10 elements or more than 30 elements.
As shown in Table~\ref{tab:cardinality_main}, SW-DRSO maintains consistent Recall@10 gains across all cardinality groups, while its training time grows moderately with set size. This suggests that the barycentric inner optimization introduces manageable overhead beyond the base SW encoder.

\begin{table}[t]
\centering
\caption{Cardinality-wise scalability on LiveJ. Cardinality groups are defined by the number of elements in each input set. R@10 denotes Recall@10, and time is measured per epoch.}
\label{tab:cardinality_main}
\resizebox{\linewidth}{!}{
\begin{tabular}{p{2.2cm}cccccccc}
\toprule
\multirow{2}{*}{\textbf{Method}}
& \multicolumn{2}{c}{$\boldsymbol{[1,10]}$}
& \multicolumn{2}{c}{$\boldsymbol{[11,20]}$}
& \multicolumn{2}{c}{$\boldsymbol{[20,30]}$}
& \multicolumn{2}{c}{$\boldsymbol{>30}$} \\
\cmidrule(lr){2-3} \cmidrule(lr){4-5} \cmidrule(lr){6-7} \cmidrule(lr){8-9}
& \textbf{R@10} & \textbf{Time}
& \textbf{R@10} & \textbf{Time}
& \textbf{R@10} & \textbf{Time}
& \textbf{R@10} & \textbf{Time} \\
\midrule
\texttt{FSW}
& 64.24 & 11.7
& 64.18 & 11.9
& 64.32 & 12.2
& 64.30 & 12.7 \\
\texttt{PSWE}
& 71.44 & 12.1
& 71.81 & 12.6
& 71.76 & 13.2
& 71.92 & 13.9 \\
\midrule
\textbf{Ours}
& \textbf{74.65} & 14.6
& \textbf{74.91} & 15.9
& \textbf{74.77} & 16.7
& \textbf{74.81} & 17.1 \\
\bottomrule
\end{tabular}
}
\end{table}

% The additional simplex weights used by the barycentric adversary are
% training-time optimization variables rather than persistent model
% parameters. They are discarded after training, and inference only
% requires a standard forward pass $g(f_{\mathrm{SW}}(S))$ without
% neighbor search or projected ascent.

\subsubsection{\textbf{Comparison with Other DRO Variants}}
We benchmark \model~against classical DRO variants, including WDRO~\cite{gao2024wasserstein}, KL-DRO~\cite{namkoong2016stochastic}, and MMD-DRO~\cite{staib2019distributionally}. 
Detailed descriptions are provided in Appendix~\ref{app:dro}.
We implement them on our encoder architecture by replacing the distance metric used in the local neighbor pool construction (Section~\ref{sec:wbs}) with their respective divergence measures. 
As shown in Table~\ref{tab:dro_variants}, our \model~performs the best while maintaining competitive training efficiency. 
WDRO and MMD-DRO suffer from significant computational overhead while delivering inferior accuracy. 
We attribute this to the geometric alignment between our Wasserstein-based DRO formulation and Sliced-Wasserstein encoder, which enables effective robust optimization, unlike the geometric mismatches in other variants.

\section{Conclusion}
\enlargethispage{2\baselineskip}

In this work, we addressed the challenge of Inference-time Element Corruption in Set Representation Learning, an issue where deployed models face sparse degradations typically absent from training data. 
We proposed \model, a scalable distributionally robust optimization framework that leverages Sliced-Wasserstein geometry to define a tractable ambiguity region. 
By introducing a barycentric adversary, we transformed the computationally prohibitive search for worst-case perturbations into an efficient optimization process. 
Extensive experiments across diverse tasks and data modalities demonstrate that \model~effectively enhances robustness against severe corruption while maintaining competitive accuracy on clean inputs. 
A promising future direction is to extend such robust set optimization to LLM- and agent-centric settings where robustness is increasingly required under distributional changes~\cite{he2026search,zhang2026coevoskills,wu2026gam,huang2025deepresearchguard,chen2026embracing,luo2026atelierevalagenticevaluationhumans,luo2026centaurevalbenchmarkinghumanintheloopvalue}.

\section*{Acknowledgements}
This work is supported in part by NSF under grants III-2106758, and POSE-2346158.

\section*{Impact Statement}
This work addresses a practical reliability gap in Set Representation Learning: element corruptions at inference time may disproportionately distort pooled set embeddings and cause downstream failures. 
To mitigate this, we propose a scalable distributionally robust optimization framework that represents sets as empirical measures and defines tractable corruption regions, enabling robustness without discrete enumeration of corrupted set variants.  
We further introduce a barycentric data synthesis strategy that turns the otherwise difficult inner worst-case search into efficient parameterization optimization.  
Based on our proposed method, we aim to improve the operational robustness of set-based models, reducing failure rates when inputs are incomplete, noisy, or adversarially perturbed.


\bibliography{ref}

@article{zaheer2017deep,
  title={Deep sets},
  author={Zaheer, Manzil and Kottur, Satwik and Ravanbakhsh, Siamak and Poczos, Barnabas and Salakhutdinov, Russ R and Smola, Alexander J},
  journal={Advances in neural information processing systems},
  volume={30},
  year={2017}
}

@article{mikolov2013efficient,
  title={Efficient estimation of word representations in vector space},
  author={Mikolov, Tomas},
  journal={ICLR},
  volume={3781},
  year={2013}
}

@article{pswe,
  title={Pooling by sliced-Wasserstein embedding},
  author={Naderializadeh, Navid and Comer, Joseph F and Andrews, Reed and Hoffmann, Heiko and Kolouri, Soheil},
  journal={NeurIPS},
  volume={34},
  pages={3389--3400},
  year={2021}
}

@inproceedings{lee2019set,
  title={Set transformer: A framework for attention-based permutation-invariant neural networks},
  author={Lee, Juho and Lee, Yoonho and Kim, Jungtaek and Kosiorek, Adam and Choi, Seungjin and Teh, Yee Whye},
  booktitle={ICML},
  pages={3744--3753},
  year={2019},
  organization={PMLR}
}

@article{agueh2011barycenters,
  title={Barycenters in the Wasserstein space},
  author={Agueh, Martial and Carlier, Guillaume},
  journal={SIAM Journal on Mathematical Analysis},
  volume={43},
  number={2},
  pages={904--924},
  year={2011},
  publisher={SIAM}
}

@article{zhang2019fspool,
  title={Fspool: Learning set representations with featurewise sort pooling},
  author={Zhang, Yan and Hare, Jonathon and Pr{\"u}gel-Bennett, Adam},
  journal={ICLR},
  year={2020}
}

@article{rahimian2019distributionally,
  title={Distributionally robust optimization: A review},
  author={Rahimian, Hamed and Mehrotra, Sanjay},
  journal={arXiv preprint arXiv:1908.05659},
  year={2019}
}

@techreport{scarf1957min,
  title={A min-max solution of an inventory problem},
  author={Scarf, Herbert E and Arrow, KJ and Karlin, S},
  year={1957},
  institution={Rand Corporation Santa Monica}
}

@article{muandet2012learning,
  title={Learning from distributions via support measure machines},
  author={Muandet, Krikamol and Fukumizu, Kenji and Dinuzzo, Francesco and Sch{\"o}lkopf, Bernhard},
  journal={Advances in neural information processing systems},
  volume={25},
  year={2012}
}

@inproceedings{edwards2017towards,
  title={Towards a Neural Statistician},
  author={Edwards, Harrison and Storkey, Amos},
  booktitle={5th International Conference on Learning Representations},
  pages={1--13},
  year={2017}
}

@book{peyre2019computational,
  title={Computational optimal transport: With applications to data science},
  author={Peyr{\'e}, Gabriel and Cuturi, Marco},
  year={2019},
  publisher={Now Foundations and Trends}
}

@book{shalev2014understanding,
  title={Understanding Machine Learning: From Theory to Algorithms},
  author={Shalev-Shwartz, Shai and Ben-David, Shai},
  year={2014},
  publisher={Cambridge University Press}
}

@article{bartlett2017spectrally,
  title={Spectrally-normalized margin bounds for neural networks},
  author={Bartlett, Peter L. and Foster, Dylan J. and Telgarsky, Matus},
  journal={Advances in Neural Information Processing Systems},
  volume={30},
  year={2017}
}

@article{cheng2017remote,
  title={Remote sensing image scene classification: Benchmark and state of the art},
  author={Cheng, Gong and Han, Junwei and Lu, Xiaoqiang},
  journal={Proceedings of the IEEE},
  volume={105},
  number={10},
  pages={1865--1883},
  year={2017},
  publisher={IEEE}
}

@article{DBLP:journals/kais/YangL15,
  author       = {Jaewon Yang and
                  Jure Leskovec},
  title        = {Defining and evaluating network communities based on ground-truth},
  journal      = {Knowl. Inf. Syst.},
  volume       = {42},
  number       = {1},
  pages        = {181--213},
  year         = {2015}
}

@inproceedings{DBLP:conf/imc/MisloveMGDB07,
  author       = {Alan Mislove and
                  Massimiliano Marcon and
                  P. Krishna Gummadi and
                  Peter Druschel and
                  Bobby Bhattacharjee},
  editor       = {Constantine Dovrolis and
                  Matthew Roughan},
  title        = {Measurement and analysis of online social networks},
  booktitle    = {Proceedings of SIGCOMM},
  pages        = {29--42}
}

@inproceedings{wu20153d,
  title={3d shapenets: A deep representation for volumetric shapes},
  author={Wu, Zhirong and Song, Shuran and Khosla, Aditya and Yu, Fisher and Zhang, Linguang and Tang, Xiaoou and Xiao, Jianxiong},
  booktitle={Proceedings of the IEEE CVPR},
  pages={1912--1920},
  year={2015}
}

@article{lin2013network,
  title={Network in network},
  author={Lin, Min and Chen, Qiang and Yan, Shuicheng},
  journal={arXiv preprint arXiv:1312.4400},
  year={2013}
}

@inproceedings{adam,
  title={Adam: A method for stochastic optimization},
  author={Kingma, Diederik P and Ba, Jimmy},
  booktitle = {ICLR},
  year      = {2015}
}

@article{levitin1966constrained,
  title={Constrained minimization methods},
  author={Levitin, Evgeny S and Polyak, Boris T},
  journal={USSR Computational mathematics and mathematical physics},
  volume={6},
  number={5},
  pages={1--50},
  year={1966},
  publisher={Elsevier}
}

@inproceedings{skianis2020rep,
  title={Rep the set: Neural networks for learning set representations},
  author={Skianis, Konstantinos and Nikolentzos, Giannis and Limnios, Stratis and Vazirgiannis, Michalis},
  booktitle={AISTATS},
  pages={1410--1420},
  year={2020},
  organization={PMLR}
}

@inproceedings{kim2022differentiable,
  title={Differentiable expectation-maximization for set representation learning},
  author={Kim, Minyoung},
  booktitle={International Conference on Learning Representations},
  year={2022}
}

@inproceedings{amirfourier,
  title={Fourier Sliced-Wasserstein Embedding for Multisets and Measures},
  author={Amir, Tal and Dym, Nadav},
  booktitle={The Thirteenth International Conference on Learning Representations},
  year={2025}
}

@article{li2021les3,
  title={LES3: learning-based exact set similarity search},
  author={Li, Yifan and Yu, Xiaohui and Koudas, Nick},
  journal={Proceedings of VLDB},
  volume={14},
  number={11},
  pages={2073--2086},
  year={2021}
}

@inproceedings{murphy2018janossy,
  title={Janossy Pooling: Learning Deep Permutation-Invariant Functions for Variable-Size Inputs},
  author={Murphy, Ryan L and Srinivasan, Balasubramaniam and Rao, Vinayak and Ribeiro, Bruno},
  booktitle={ICLR},
  year={2018}
}

@inproceedings{zhanglearning,
  title={Learning Representations of Sets through Optimized Permutations},
  author={Zhang, Yan and Hare, Jonathon and Pr{\"u}gel-Bennett, Adam},
  booktitle={ICLR},
  year={2019}
}

@article{rezatofighi2018deep,
  title={Deep perm-set net: Learn to predict sets with unknown permutation and cardinality using deep neural networks},
  author={Rezatofighi, S Hamid and Kaskman, Roman and Motlagh, Farbod T and Shi, Qinfeng and Cremers, Daniel and Leal-Taix{\'e}, Laura and Reid, Ian},
  journal={arXiv preprint arXiv:1805.00613},
  year={2018}
}

@article{vaswani2017attention,
  title={Attention is all you need},
  author={Vaswani, Ashish and Shazeer, Noam and Parmar, Niki and Uszkoreit, Jakob and Jones, Llion and Gomez, Aidan N and Kaiser, {\L}ukasz and Polosukhin, Illia},
  journal={NeurIPS},
  volume={30},
  year={2017}
}

@inproceedings{jaegle2021perceiver,
  title={Perceiver: General perception with iterative attention},
  author={Jaegle, Andrew and Gimeno, Felix and Brock, Andy and Vinyals, Oriol and Zisserman, Andrew and Carreira, Joao},
  booktitle={ICML},
  pages={4651--4664},
  year={2021},
  organization={PMLR}
}

@inproceedings{mialon2021trainable,
  title={A Trainable Optimal Transport Embedding for Feature Aggregation and its Relationship to Attention},
  author={Mialon, Gr{\'e}goire and Chen, Dexiong and d'Aspremont, Alexandre and Mairal, Julien},
  booktitle={ICLR},
  year={2021}
}

@inproceedings{lee2022set2box,
  title={Set2Box: Similarity Preserving Representation Learning for Sets},
  author={Lee, Geon and Park, Chanyoung and Shin, Kijung},
  booktitle={ICDM},
  pages={1023--1028},
  year={2022},
  organization={IEEE}
}

@article{xu2025fuse,
  title={FUSE: Measure-theoretic compact fuzzy set representation for taxonomy expansion},
  author={Xu, Fred and Jiang, Song and Huang, Zijie and Luo, Xiao and Zhang, Shichang and Chen, Adrian and Sun, Yizhou},
  journal={arXiv preprint arXiv:2506.08409},
  year={2025}
}

@article{guo2021learning,
  title={Learning prototype-oriented set representations for meta-learning},
  author={Guo, Dandan and Tian, Long and Zhang, Minghe and Zhou, Mingyuan and Zha, Hongyuan},
  journal={arXiv preprint arXiv:2110.09140},
  year={2021}
}

@inproceedings{leeself,
  title={Self-Supervised Set Representation Learning for Unsupervised Meta-Learning},
  author={Lee, Dong Bok and Lee, Seanie and Kawaguchi, Kenji and Kim, Yunji and Bang, Jihwan and Ha, Jung-Woo and Hwang, Sung Ju},
  booktitle={ICLR},
  year={2023}
}

@article{vargas2025quantum,
  title={Quantum deep sets and sequences},
  author={Vargas-Calder{\'o}n, Vladimir},
  journal={Quantum Machine Intelligence},
  volume={7},
  number={2},
  pages={65},
  year={2025},
  publisher={Springer}
}

@article{naderializadeh2025aggregating,
  title={Aggregating residue-level protein language model embeddings with optimal transport},
  author={NaderiAlizadeh, Navid and Singh, Rohit},
  journal={Bioinformatics Advances},
  volume={5},
  number={1},
  pages={vbaf060},
  year={2025},
  publisher={Oxford University Press}
}

@incollection{kuhn2019wasserstein,
  title={Wasserstein distributionally robust optimization: Theory and applications in machine learning},
  author={Kuhn, Daniel and Esfahani, Peyman Mohajerin and Nguyen, Viet Anh and Shafieezadeh-Abadeh, Soroosh},
  booktitle={Operations research \& management science in the age of analytics},
  pages={130--166},
  year={2019},
  publisher={Informs}
}

@article{rahimian2022frameworks,
  title={Frameworks and results in distributionally robust optimization},
  author={Rahimian, Hamed and Mehrotra, Sanjay},
  journal={Open Journal of Mathematical Optimization},
  volume={3},
  pages={1--85},
  year={2022}
}

@article{bertsimas2019adaptive,
  title={Adaptive distributionally robust optimization},
  author={Bertsimas, Dimitris and Sim, Melvyn and Zhang, Meilin},
  journal={Management Science},
  volume={65},
  number={2},
  pages={604--618},
  year={2019},
  publisher={INFORMS}
}

@article{wang2025knowledge,
  title={Knowledge-Guided Wasserstein Distributionally Robust Optimization},
  author={Wang, Zitao and Wang, Ziyuan and Liu, Molei and Si, Nian},
  booktitle={Forty-second International Conference on Machine Learning},
  year={2025}
}

@inproceedings{nguyenbeyond,
  title={Beyond Minimax Rates in Group Distributionally Robust Optimization via a Novel Notion of Sparsity},
  author={Nguyen, Quan M and Mehta, Nishant A and Guzm{\'a}n, Crist{\'o}bal A},
  booktitle={Forty-second International Conference on Machine Learning},
  year={2025}
}

@inproceedings{ma2024differentiable,
  title={Differentiable Distributionally Robust Optimization Layers},
  author={Ma, Xutao and Ning, Chao and Du, Wenli},
  booktitle={International Conference on Machine Learning},
  pages={33880--33901},
  year={2024},
  organization={PMLR}
}

@inproceedings{yu2024efficient,
  title={Efficient Algorithms for Empirical Group Distributionally Robust Optimization and Beyond},
  author={Yu, Dingzhi and Cai, Yunuo and Jiang, Wei and Zhang, Lijun},
  booktitle={International Conference on Machine Learning},
  pages={57384--57414},
  year={2024},
  organization={PMLR}
}

@inproceedings{inatsu2022bayesian,
  title={Bayesian optimization for distributionally robust chance-constrained problem},
  author={Inatsu, Yu and Takeno, Shion and Karasuyama, Masayuki and Takeuchi, Ichiro},
  booktitle={International Conference on Machine Learning},
  pages={9602--9621},
  year={2022},
  organization={PMLR}
}

@article{husain2023distributionally,
  title={Distributionally Robust Bayesian Optimization with $\phi$-divergences},
  author={Husain, Hisham and Nguyen, Vu and van den Hengel, Anton},
  journal={Advances in Neural Information Processing Systems},
  volume={36},
  pages={20133--20145},
  year={2023}
}

@article{zhou2020planning,
  title={Planning PEV fast-charging stations using data-driven distributionally robust optimization approach based on $\phi$-divergence},
  author={Zhou, Bo and Chen, Guo and Huang, Tingwen and Song, Qiankun and Yuan, Yuefei},
  journal={IEEE Transactions on Transportation Electrification},
  volume={6},
  number={1},
  pages={170--180},
  year={2020},
  publisher={IEEE}
}

@article{gao2024wasserstein,
  title={Wasserstein distributionally robust optimization and variation regularization},
  author={Gao, Rui and Chen, Xi and Kleywegt, Anton J},
  journal={Operations Research},
  volume={72},
  number={3},
  pages={1177--1191},
  year={2024},
  publisher={Informs}
}

@inproceedings{kwon2020principled,
  title={Principled learning method for wasserstein distributionally robust optimization with local perturbations},
  author={Kwon, Yongchan and Kim, Wonyoung and Won, Joong-Ho and Paik, Myunghee Cho},
  booktitle={International Conference on Machine Learning},
  pages={5567--5576},
  year={2020},
  organization={PMLR}
}

@inproceedings{micheli2025wasserstein,
  title={Wasserstein Distributionally Robust Bayesian Optimization with Continuous Context},
  author={Micheli, Francesco and Balta, Efe C and Tsiamis, Anastasios and Lygeros, John},
  booktitle={International Conference on Artificial Intelligence and Statistics},
  pages={4978--4986},
  year={2025},
  organization={PMLR}
}

@inproceedings{zhang2025revisiting,
  title={Revisiting large-scale non-convex distributionally robust optimization},
  author={Zhang, Qi and Zhou, Yi and Khan, Simon and Prater-Bennette, Ashley and Shen, Lixin and Zou, Shaofeng},
  booktitle={The Thirteenth International Conference on Learning Representations},
  year={2025}
}

@article{sagawa2019distributionally,
  title={Distributionally robust neural networks for group shifts: On the importance of regularization for worst-case generalization},
  author={Sagawa, Shiori and Koh, Pang Wei and Hashimoto, Tatsunori B and Liang, Percy},
  journal={arXiv preprint arXiv:1911.08731},
  year={2019}
}

@article{namkoong2016stochastic,
  title={Stochastic gradient methods for distributionally robust optimization with f-divergences},
  author={Namkoong, Hongseok and Duchi, John C},
  journal={Advances in neural information processing systems},
  volume={29},
  year={2016}
}

@article{staib2019distributionally,
  title={Distributionally robust optimization and generalization in kernel methods},
  author={Staib, Matthew and Jegelka, Stefanie},
  journal={Advances in Neural Information Processing Systems},
  volume={32},
  year={2019}
}

@article{cuturi2013sinkhorn,
  title={Sinkhorn distances: Lightspeed computation of optimal transport},
  author={Cuturi, Marco},
  journal={Advances in neural information processing systems},
  volume={26},
  year={2013}
}

@inproceedings{pele2009fast,
  title={Fast and robust earth mover distances},
  author={Pele, Ofir and Werman, Michael},
  booktitle={2009 IEEE 12th International Conference on Computer Vision},
  pages={460--467},
  year={2009},
  organization={IEEE}
}

@inproceedings{qi2017pointnet,
  title={Pointnet: Deep learning on point sets for 3d classification and segmentation},
  author={Qi, Charles R and Su, Hao and Mo, Kaichun and Guibas, Leonidas J},
  booktitle={Proceedings of the IEEE CVPR},
  pages={652--660},
  year={2017}
}

@article{kolouri2019generalized,
  title={Generalized sliced wasserstein distances},
  author={Kolouri, Soheil and Nadjahi, Kimia and Simsekli, Umut and Badeau, Roland and Rohde, Gustavo},
  journal={NeurIPS},
  volume={32},
  year={2019}
}

@article{bonneel2015sliced,
  title={Sliced and radon wasserstein barycenters of measures},
  author={Bonneel, Nicolas and Rabin, Julien and Peyr{\'e}, Gabriel and Pfister, Hanspeter},
  journal={JMIV},
  volume={51},
  number={1},
  pages={22--45},
  year={2015},
  publisher={Springer}
}

@article{rusu2008towards,
  title={Towards 3D point cloud based object maps for household environments},
  author={Rusu, Radu Bogdan and Marton, Zoltan Csaba and Blodow, Nico and Dolha, Mihai and Beetz, Michael},
  journal={Robotics and Autonomous Systems},
  volume={56},
  number={11},
  pages={927--941},
  year={2008},
  publisher={Elsevier}
}

@inproceedings{zhou2022understanding,
  title={Understanding the robustness in vision transformers},
  author={Zhou, Daquan and Yu, Zhiding and Xie, Enze and Xiao, Chaowei and Anandkumar, Animashree and Feng, Jiashi and Alvarez, Jose M},
  booktitle={International conference on machine learning},
  pages={27378--27394},
  year={2022},
  organization={PMLR}
}

@article{kantorovich2006translocation,
  title={On the translocation of masses},
  author={Kantorovich, L. V.},
  journal={Journal of mathematical sciences},
  volume={133},
  number={4},
  year={2006}
}

@inproceedings{liu2025provable,
  title={Provable robust overfitting mitigation in Wasserstein distributionally robust optimization},
  author={Liu, S. and Wang, Y. and Zhu, Y. and Miao, Y. and Gao, X.-S.},
  booktitle={The Thirteenth International Conference on Learning Representations},
  year={2025}
}

@inproceedings{huang2025effective,
  title={An effective manifold-based optimization method for distributionally robust classification},
  author={Huang, J. and Ding, H.},
  booktitle={The Thirteenth International Conference on Learning Representations},
  year={2025}
}

@inproceedings{wurobust,
  title={Towards robust alignment of language models: Distributionally robustifying direct preference optimization},
  author={Wu, J. and Xie, Y. and Yang, Z. and Wu, J. and Chen, J. and Gao, J. and Ding, B. and Wang, X. and He, X.},
  booktitle={The Thirteenth International Conference on Learning Representations},
  year={2025}
}

@article{zhang2022knowledge,
  title={Knowledge-aware neural networks with personalized feature referencing for cold-start recommendation},
  author={Zhang, Xinni and Chen, Yankai and Gao, Cuiyun and Liao, Qing and Zhao, Shenglin and King, Irwin},
  journal={arXiv preprint arXiv:2209.13973},
  year={2022}
}

@inproceedings{chenadversarial,
  title={Adversarial Encoding Perturbation and Synthesis for Set Representation Auxiliary Learning},
  author={Chen, Yankai and Zhang, Xinni and Zou, Henry Peng and He, Bowei and Li, Yangning and Yu, Philip S and King, Irwin and Liu, Xue},
  booktitle={The Fourteenth International Conference on Learning Representations},
  year={2026}
}

@article{chen2026embracing,
  title={Embracing Trustworthy Brain-Agent Collaboration as Paradigm Extension for Intelligent Assistive Technologies},
  author={Chen, Yankai and Zhang, Xinni and Zhang, Yifei and Li, Yangning and Zou, Henry and Miao, Chunyu and Zhang, Weizhi and Liu, Steve Xue and Yu, Philip S},
  journal={Advances in Neural Information Processing Systems},
  volume={38},
  year={2026}
}

@inproceedings{chen2023star,
title={Topological Representation Learning for E-commerce Shopping Behaviors},
author={Chen, Yankai and Truong, Quoc-Tuan and Shen, Xin and Wang, Ming and Li, Jin and Chan, Jim and King, Irwin},
journal={Proceedings of the 19th International Workshop on Mining and Learning with Graphs (MLG)},
year={2023}
}

@article{huang2025deepresearchguard,
  title={Deepresearchguard: Deep research with open-domain evaluation and multi-stage guardrails for safety},
  author={Huang, Wei-Chieh and Zou, Henry Peng and Wu, Yaozu and Li, Dongyuan and Chen, Yankai and Zhang, Weizhi and Li, Yangning and Zangari, Angelo and Guo, Jizhou and Miao, Chunyu and others},
  journal={arXiv preprint arXiv:2510.10994},
  year={2025}
}

@article{hihpq,
  title={HiHPQ: Hierarchical Hyperbolic Product Quantization for Unsupervised Image Retrieval}, 
  author={Qiu, Zexuan and Liu, Jiahong and Chen, Yankai and King, Irwin},
  journal={Proceedings of the 38th AAAI Conference on Artificial Intelligence (AAAI)},
  year={2024}
}

@inproceedings{chen2024deep,
  title={Deep Structural Knowledge Exploitation and Synergy for Estimating Node Importance Value on Heterogeneous Information Networks},
  author={Chen, Yankai and Fang, Yixiang and Wang, Qiongyan and Cao, Xin and King, Irwin},
  booktitle={Proceedings of the AAAI Conference on Artificial Intelligence},
  volume={38},
  number={8},
  pages={8302--8310},
  year={2024}
}

@article{zhang2026coevoskills,
  title={Coevoskills: Self-evolving agent skills via co-evolutionary verification},
  author={Zhang, Hanrong and Fan, Shicheng and Zou, Henry Peng and Chen, Yankai and Wang, Zhenting and Zhou, Jiayu and Li, Chengze and Huang, Wei-Chieh and Yao, Yifei and Zheng, Kening and others},
  journal={arXiv preprint arXiv:2604.01687},
  year={2026}
}

@inproceedings{chen2022modeling,
  title={Modeling scale-free graphs with hyperbolic geometry for knowledge-aware recommendation},
  author={Chen, Yankai and Yang, Menglin and Zhang, Yingxue and Zhao, Mengchen and Meng, Ziqiao and Hao, Jianye and King, Irwin},
  booktitle={Proceedings of the fifteenth ACM international conference on web search and data mining},
  pages={94--102},
  year={2022}
}

@inproceedings{chen2022attentive,
  title={Attentive knowledge-aware graph convolutional networks with collaborative guidance for personalized recommendation},
  author={Chen, Yankai and Yang, Yaming and Wang, Yujing and Bai, Jing and Song, Xiangchen and King, Irwin},
  booktitle={2022 IEEE 38th international conference on data engineering (ICDE)},
  pages={299--311},
  year={2022},
  organization={IEEE}
}

@inproceedings{he2023dynamically,
  title={Dynamically expandable graph convolution for streaming recommendation},
  author={He, Bowei and He, Xu and Zhang, Yingxue and Tang, Ruiming and Ma, Chen},
  booktitle={Proceedings of the ACM web conference 2023},
  pages={1457--1467},
  year={2023}
}

@inproceedings{he2023dynamic,
  title={Dynamic embedding size search with minimum regret for streaming recommender system},
  author={He, Bowei and He, Xu and Zhang, Renrui and Zhang, Yingxue and Tang, Ruiming and Ma, Chen},
  booktitle={Proceedings of the 32nd ACM International Conference on Information and Knowledge Management},
  pages={741--750},
  year={2023}
}

@article{he2026search,
  title={Search-R2: Enhancing Search-Integrated Reasoning via Actor-Refiner Collaboration},
  author={He, Bowei and Hu, Minda and Xu, Zenan and Wang, Hongru and Zong, Licheng and Chen, Yankai and Ma, Chen and Liu, Xue and Zhou, Pluto and King, Irwin},
  journal={arXiv preprint arXiv:2602.03647},
  year={2026}
}

@article{he2026preserving,
  title={Preserving llm capabilities through calibration data curation: From analysis to optimization},
  author={He, Bowei and Yin, Lihao and Zhen, Hui-Ling and Liu, Shuqi and Wu, Han and Zhang, Xiaokun and Yuan, Mingxuan and Ma, Chen},
  journal={Advances in Neural Information Processing Systems},
  volume={38},
  pages={58531--58572},
  year={2026}
}

@inproceedings{he2024interpretable,
  title={Interpretable triplet importance for personalized ranking},
  author={He, Bowei and Ma, Chen},
  booktitle={Proceedings of the 33rd ACM International Conference on Information and Knowledge Management},
  pages={809--818},
  year={2024}
}

@article{he2026pedagogically,
  title={Pedagogically-Inspired Data Synthesis for Language Model Knowledge Distillation},
  author={He, Bowei and Chen, Yankai and Zhang, Xiaokun and Kong, Linghe and Yu, Philip S and Liu, Xue and Ma, Chen},
  journal={arXiv preprint arXiv:2602.12172},
  year={2026}
}

@article{he2025paser,
  title={PASER: Post-training data selection for efficient pruned large language model recovery},
  author={He, Bowei and Yin, Lihao and Zhen, Hui-Ling and Zhang, Xiaokun and Yuan, Mingxuan and Ma, Chen},
  journal={arXiv preprint arXiv:2502.12594},
  year={2025}
}

@inproceedings{chen2023sim2rec,
  title={Sim2rec: A simulator-based decision-making approach to optimize real-world long-term user engagement in sequential recommender systems},
  author={Chen, Xiong-Hui and He, Bowei and Yu, Yang and Li, Qingyang and Qin, Zhiwei and Shang, Wenjie and Ye, Jieping and Ma, Chen},
  booktitle={2023 IEEE 39th International Conference on Data Engineering (ICDE)},
  pages={3389--3402},
  year={2023},
  organization={IEEE}
}

@article{luo2025recranker,
  title={Recranker: Instruction tuning large language model as ranker for top-k recommendation},
  author={Luo, Sichun and He, Bowei and Zhao, Haohan and Shao, Wei and Qi, Yanlin and Huang, Yinya and Zhou, Aojun and Yao, Yuxuan and Li, Zongpeng and Xiao, Yuanzhang and others},
  journal={ACM Transactions on Information Systems},
  volume={43},
  number={5},
  pages={1--31},
  year={2025},
  publisher={ACM New York, NY}
}

@article{luo2026integrating,
  title={Integrating large language models into recommendation via mutual augmentation and adaptive aggregation},
  author={Luo, Sichun and Yao, Yuxuan and He, Bowei and Shao, Wei and Xu, Jian and Huang, Yinya and Zhou, Aojun and Zhang, Xinyi and Xiao, Yuanzhang and Hou, Hanxu and others},
  journal={IEEE Journal of Selected Topics in Signal Processing},
  year={2026},
  publisher={IEEE}
}

@inproceedings{cui2024context,
  title={Context matters: Enhancing sequential recommendation with context-aware diffusion-based contrastive learning},
  author={Cui, Ziqiang and Wu, Haolun and He, Bowei and Cheng, Ji and Ma, Chen},
  booktitle={Proceedings of the 33rd ACM International Conference on Information and Knowledge Management},
  pages={404--414},
  year={2024}
}

@article{cui2026semantic,
  title={Semantic retrieval augmented contrastive learning for sequential recommendation},
  author={Cui, Ziqiang and Weng, Yunpeng and Tang, Xing and Zhang, Xiaokun and Li, Shiwei and Liu, Peiyang and He, Bowei and Liu, Dugang and Luo, Weihong and He, Xiuqiang and others},
  journal={Advances in Neural Information Processing Systems},
  volume={38},
  pages={162621--162645},
  year={2026}
}

@inproceedings{zhang2026token,
  title={From Token to Item: Enhancing Large Language Models for Recommendation via Item-aware Attention Mechanism},
  author={Zhang, Xiaokun and He, Bowei and Chen, Jiamin and Cui, Ziqiang and Ma, Chen},
  booktitle={Proceedings of the ACM Web Conference 2026},
  pages={6700--6708},
  year={2026}
}

@inproceedings{wu2026gam,
  title={GAM: Hierarchical Graph-based Agentic Memory for LLM Agents},
  author={Wu, Zhaofen and Zhang, Hanrong and Lin, Fulin and Xu, Wujiang and Xu, Xinran and Chen, Yankai and Zou, Henry Peng and Chen, Shaowen and Zhang, Weizhi and Liu, Xue and others},
  booktitle={Proceedings of the 64th Annual Meeting of the Association for Computational Linguistics (ACL)},
  year={2026}
}

@article{li2026dyg,
  title={Dyg-mamba: Continuous state space modeling on dynamic graphs},
  author={Li, Dongyuan and Tan, Shiyin and Zhang, Ying and Jin, Ming and Pan, Shirui and Okumura, Manabu and Jiang, Renhe},
  journal={Advances in Neural Information Processing Systems},
  volume={38},
  pages={129101--129130},
  year={2026}
}

@inproceedings{lirouting,
  title={Routing Channel-Patch Dependencies in Time Series Forecasting with Graph Spectral Decomposition},
  author={Li, Dongyuan and Zheng, Shun and Xu, Chang and Bian, Jiang and Jiang, Renhe},
  booktitle={The Fourteenth International Conference on Learning Representations},
  year={2026}
}

@inproceedings{li2026node,
  title={Node Role-Guided LLMs for Dynamic Graph Clustering},
  author={Li, Dongyuan and Zhang, Ying and Wu, Yaozu and Jiang, Renhe},
  booktitle={Proceedings of the ACM Web Conference 2026},
  pages={4418--4428},
  year={2026}
}

@misc{luo2026atelierevalagenticevaluationhumans,
  title={AtelierEval: Agentic Evaluation of Humans \& LLMs as Text-to-Image Prompters},
  author={Luo, Hanjun and Huang, Zhimu and Chung, Sylvia and Wang, Yiran and Jin, Yingbin and Li, Jialin and Li, Jiang and Li, Xinfeng and Salam, Hanan},
  year={2026},
  eprint={2605.22645},
  archivePrefix={arXiv},
  primaryClass={cs.AI},
  url={https://arxiv.org/abs/2605.22645}
}

@misc{luo2026centaurevalbenchmarkinghumanintheloopvalue,
  title={CentaurEval: Benchmarking Human-in-the-Loop Value in Agentic Coding},
  author={Luo, Hanjun and Ni, Chiming and Wen, Jiaheng and Huang, Zhimu and Wang, Yiran and Liao, Bingduo and Chung, Sylvia and Jin, Yingbin and Li, Xinfeng and Xu, Wenyuan and Wang, XiaoFeng and Salam, Hanan},
  year={2026},
  eprint={2512.04111},
  archivePrefix={arXiv},
  primaryClass={cs.SE},
  url={https://arxiv.org/abs/2512.04111}
}

@article{luo2026agentauditor,
  title={AgentAuditor: Human-Level Safety and Security Evaluation for LLM Agents},
  author={Luo, Hanjun and Dai, Shenyu and Ni, Chiming and Li, Xinfeng and Zhang, Guibin and Wang, Kun and Liu, Tongliang and Salam, Hanan},
  journal={Advances in Neural Information Processing Systems},
  volume={38},
  pages={43241--43298},
  year={2026}
}

@inproceedings{luo2025dynamicner,
  title={DynamicNER: A Dynamic, Multilingual, and Fine-Grained Dataset for LLM-based Named Entity Recognition},
  author={Luo, Hanjun and Jin, Yingbin and Wang, Yiran and Li, Xinfeng and Shang, Tong and Liu, Xuecheng and Chen, Ruizhe and Wang, Kun and Salam, Hanan and Wen, Qingsong and Liu, Zuozhu},
  booktitle={Proceedings of the 2025 Conference on Empirical Methods in Natural Language Processing},
  year={2025}
}

@article{li2026prefix,
  title={PrefIx: Understand and Adapt to User Preference in Human-Agent Interaction},
  author={Li, Jialin and Chen, Zhenhao and Luo, Hanjun and Salam, Hanan},
  journal={arXiv preprint arXiv:2602.06714},
  year={2026}
}

@article{luo2024versusdebias,
  title={VersusDebias: Universal zero-shot debiasing for text-to-image models via SLM-based prompt engineering and generative adversary},
  author={Luo, Hanjun and Deng, Ziye and Huang, Haoyu and Liu, Xuecheng and Chen, Ruizhe and Liu, Zuozhu},
  journal={arXiv preprint arXiv:2407.19524},
  year={2024}
}
\bibliographystyle{icml2026}

%%%%%%%%%%%%%%%%%%%%%%%%%%%%%%%%%%%%%%%%%%%%%%%%%%%%%%%%%%%%%%%%%%%%%%%%%%%%%%%
%%%%%%%%%%%%%%%%%%%%%%%%%%%%%%%%%%%%%%%%%%%%%%%%%%%%%%%%%%%%%%%%%%%%%%%%%%%%%%%
% APPENDIX
%%%%%%%%%%%%%%%%%%%%%%%%%%%%%%%%%%%%%%%%%%%%%%%%%%%%%%%%%%%%%%%%%%%%%%%%%%%%%%%
%%%%%%%%%%%%%%%%%%%%%%%%%%%%%%%%%%%%%%%%%%%%%%%%%%%%%%%%%%%%%%%%%%%%%%%%%%%%%%%
% \newpage
\appendix
\onecolumn
\setcounter{equation}{0}

\section{Notation Explanations}
\label{app:notation}

All notations are summarized and explained in Table~\ref{tab:notation}.

{\small
\begin{longtable}{ll}
\caption{Summary of notation used throughout the paper.}
\label{tab:notation}\\
\toprule
\textbf{Symbol} & \textbf{Description} \\
\midrule
\endfirsthead
\toprule
\textbf{Symbol} & \textbf{Description} \\
\midrule
\endhead
\bottomrule
\endfoot

$\mathcal{X} \subseteq \mathbb{R}^d$, $\mathcal{Y}$ 
& Element space of individual set elements and label space\\

$x_i \in \mathcal{X}$ 
& The $i$-th element in a set \\

$S = \{x_i\}_{i=1}^n$, $S'$ 
& Unordered input set with cardinality $n$ and corrupted version of set $S$ at inference time \\

$\mathcal{S}(\mathcal{X})$ 
& Space of all finite subsets of $\mathcal{X}$ \\

\midrule

$f : \mathcal{S}(\mathcal{X}) \rightarrow \mathbb{R}^c$ 
& Permutation-invariant set encoder \\

$v_S = f(S)$ 
& Learned representation (embedding) of set $S$ \\

$g(\cdot)$ 
& Task-specific predictor (e.g., classifier or ranking function) \\

$\ell(\cdot,\cdot)$ 
& Task loss function \\

$P_0$ 
& Nominal training distribution over clean sets \\

\midrule

$\Gamma(S)$ 
& Ambiguity (corruption) region associated with set $S$ \\

$\Gamma_{\mathrm{SW}}(S)$ 
& Sliced-Wasserstein ambiguity region centered at $S$ \\

$\Gamma_{\mathrm{Bar}}(S)$ 
& Barycentric surrogate ambiguity region \\

$\rho$ 
& Radius of the ambiguity region \\

\midrule

$\mu_S$ 
& Empirical measure associated with set $S$ \\

$\delta_x$ 
& Dirac measure at location $x$ \\

$\mathcal{P}(\mathcal{X})$ 
& Space of probability measures over $\mathcal{X}$ \\

$\mu_{S'}$ 
& Empirical measure of corrupted set $S'$ \\

\midrule

$\mathbb{S}^{d-1}$ 
& Unit sphere in $\mathbb{R}^d$ \\

$\omega \in \mathbb{S}^{d-1}$ 
& Projection direction \\

$\omega(x) = \omega^\top x$ 
& One-dimensional projection of element $x$ \\

$\mu_S^\omega$ 
& Projected 1D measure of $\mu_S$ along $\omega$ \\

$W(\cdot,\cdot)$ 
& 2-Wasserstein distance \\

$\mathrm{SW}(\mu_S,\mu_{S'})$ 
& Sliced-Wasserstein distance between measures \\

\midrule

$O = \{o_h\}_{h=1}^H$ 
& Learnable reference set \\

$H$, $R$ 
& Cardinality of the reference set and number of Monte Carlo projection directions\\

$\mu_O$ 
& Empirical measure of the reference set \\

$F_\mu$, $F_\mu^{-1}$ 
& Cumulative distribution function (CDF) of measure $\mu$ and quantile function of measure $\mu$\\

$p^+(\cdot)$ 
& Optimal 1D Wasserstein transport map \\

$\Omega = \{\omega_r\}_{r=1}^R$ 
& Set of projection directions \\

$f_{\mathrm{SW}}(S)$ 
& Sliced-Wasserstein-based set embedding \\

\midrule

$\mathcal{B}$ 
& Mini-batch of training sets \\

$\mathcal{G}(S)$ 
& Local neighbor pool of set $S$ \\

$K$ 
& Number of neighbors in $\mathcal{G}(S)$ \\

$S_k$ 
& The $k$-th neighboring set of $S$ \\

$v_{S_k}$ 
& Embedding of neighbor set $S_k$ \\

$\Delta_K$ 
& Probability simplex in $\mathbb{R}^K$ \\

$\Lambda = (\lambda_1,\dots,\lambda_K)$ 
& Barycentric mixing weights \\

$\bar v_S(\Lambda)$ 
& Synthesized barycentric embedding of set $S$ \\

$S_\Lambda$ 
& Implicit (virtual) barycentric set induced by $\Lambda$ \\

\midrule

$\eta$, $\alpha$ 
& Step size for projected gradient ascent, and trade-off coefficient between ERM and robust loss \\

$T$ 
& Number of inner maximization steps \\

$\Pi_{\Delta_K}(\cdot)$ 
& Euclidean projection onto simplex $\Delta_K$ \\

$\Lambda^{(t)}$, $\Lambda^\ast$ 
& Mixing weights at iteration $t$ and final adversarial barycentric weights \\
\end{longtable}
}

\section{Theoretical Analysis}
\label{app:proof}

\subsection{Proof of Proposition~\ref{prop:jensen_locality}}
\begin{proof}
Let $\delta_k := v_{S_k}-v_S$. 
Since $\bar v_S(\Lambda)-v_S=\sum_{k=1}^K \lambda_k \delta_k$ and the function $\phi(x)=\|x\|_2^2$ is convex, Jensen's inequality directly gives:
\begin{equation}
\big\|\bar v_S(\Lambda)-v_S\big\|_2^2 = \Big\|\sum_{k=1}^K \lambda_k \delta_k\Big\|_2^2 \le \sum_{k=1}^K \lambda_k \|\delta_k\|_2^2.
\end{equation}
Since $\|\delta_k\|_2^2=\|v_{S_k}-v_S\|_2^2\le \rho^2$ for all $k$ (the selection principle of $\mathcal{G}$) and $\sum_k\lambda_k=1$, we obtain
\begin{equation}
\sum_{k=1}^K \lambda_k \|\delta_k\|_2^2 \le \sum_{k=1}^K \lambda_k \rho^2 = \rho^2.
\end{equation}
Furthermore, according to \cite{pswe}, we have:
\begin{equation}
\label{eq:xxx}
\|\bar v_S(\Lambda) - v_S\|_2 = SW(\mu_{S(\Lambda)}, \mu_S).
\end{equation}
We provide the derivation of Eq.~\eqref{eq:xxx} for completeness in Appendix~\ref{app:eqxxx}.
This implies that:
\begin{equation}
SW(\mu_{S(\Lambda)}, \mu_S) \leq \rho,
\end{equation}
which completes the proof.
\end{proof}

\subsection{Proof of Proposition~\ref{prop:barycenter_semantics}}
\label{app:proof_barycenter_semantics}
\begin{proof}
For each direction $\omega$, the 1D Wasserstein barycenter is formulated as:
\begin{equation}
\mu_\Lambda^\omega \in \arg\min_{\mu\in\mathcal{P}(\mathbb{R})}\ \sum_{k=1}^K \lambda_k\,W_2^2(\mu,\mu_{S_k}^\omega).
\end{equation}
Then by optimality of $\mu_\Lambda^\omega$, for any candidate $\mu$ and every $\omega$, we have:
\begin{equation}
\sum_{k=1}^K \lambda_k\,W_2^2(\mu^\omega,\mu_{S_k}^\omega) \ \ge \ \sum_{k=1}^K \lambda_k\,W_2^2(\mu_\Lambda^\omega,\mu_{S_k}^\omega).
\end{equation}
Considering that, (as shown in Eq.~\eqref{eq:sliced_wasserstein}), 
\begin{equation}
SW^2(\mu,\nu)=\int_{\mathbb{S}^{d-1}} W_2^2(\mu^\omega,\nu^\omega)\,d\omega =\mathbb{E}[W_2^2(\mu^\omega,\nu^\omega)],
\end{equation}
by integrating over $\omega$, we then have:
\begin{equation}
\sum_{k=1}^K \lambda_k\,SW^2(\mu,\mu_{S_k}) \ \ge\ \sum_{k=1}^K \lambda_k\,SW^2(\mu_\Lambda,\mu_{S_k}).
\end{equation}
This implies that $\mu_\Lambda$ is an SW Fr\'echet mean:
\begin{equation}
\mu_\Lambda\in\arg\min_{\mu\in\mathcal{P}(\mathcal{X})}\ \sum_{k=1}^K \lambda_k\,SW^2(\mu,\mu_{S_k}).
\end{equation}

Finally, we relate $\mu_\Lambda$ to the synthesized embedding.
We slightly abuse notation by letting $\f_{\mathrm{SW}}$ act on a measure through its slice-wise reference-to-input OT coordinates.
Then given any slice $\omega$ and reference coordinate $t$, let $u:=F_{\mu_O^\omega}(t)\in(0,1)$.
In 1D, the $W_2$ barycenter has the quantile characterization~\cite{agueh2011barycenters,peyre2019computational}:
\begin{equation}
F^{-1}_{\mu_\Lambda^\omega}(u)=\sum_{k=1}^K \lambda_k\,F^{-1}_{\mu_{S_k}^\omega}(u).
\label{eq:quantile_linear_compact}
\end{equation}
Hence, the reference-to-input OT coordinate $\texttt{p}^+(t,\nu)=F_\nu^{-1}(F_{\mu_O^\omega}(t))=F_\nu^{-1}(u)$ is linear in $\Lambda$:
\begin{equation}
\texttt{p}^+(t,\mu_\Lambda^\omega)=\sum_{k=1}^K \lambda_k\,\texttt{p}^+(t,\mu_{S_k}^\omega).
\label{eq:pplus_linear_compact}
\end{equation}
Concatenating Eq.~\eqref{eq:pplus_linear_compact} over the sampled slices as in $\f_{\mathrm{SW}}$ finally yields:
\begin{equation}
\f_{\mathrm{SW}}(\mu_\Lambda)
=\sum_{k=1}^K \lambda_k\,\f_{\mathrm{SW}}(\mu_{S_k})
=\sum_{k=1}^K \lambda_k\,v_{S_k}
=\bar v_S(\Lambda),
\end{equation}
which completes the proof.
\end{proof}

\subsection{Proof of Proposition~\ref{prop:disc_le_bar}}
\label{app:smaller}
\begin{proof}
Recall that the barycentric region is defined as the convex hull of the neighbor embeddings:
\begin{equation}
\Gamma_{\mathrm{Bar}}(S) :=\left\{\bar v_S(\Lambda)=\sum_{k=1}^K \lambda_k v_{S_k}: \Lambda \in\Delta_K\right\},
\end{equation}
where $\Delta_K=\{\lambda\in\mathbb{R}^K:\sum_{k=1}^K\lambda_k=1, \lambda_k\ge 0\}$ is the probability simplex. 
The discrete and barycentric inner objectives are defined as $L_{\mathrm{Disc}}(S)$ $:=$ $\max_{k\in[K]} \ell\big(g(v_{S_k}),y\big)$, and 
$L_{\mathrm{Bar}}(S)$ $:=$ $\max_{v\in\Gamma_{\mathrm{Bar}}(S)} \ell\big(g(v),y\big)$.
Since $\Gamma_{\mathrm{Bar}}(S)$ is the convex hull of $\{v_{S_k}\}_{k=1}^K$, each vertex $v_{S_k}$ itself is contained in $\Gamma_{\mathrm{Bar}}(S)$ by choosing $\Lambda$ such that $\lambda_k = 1$ and $\lambda_j = 0$ for all $j \neq k$, we recover $\bar v_S(\Lambda)=v_{S_k}$.
 Therefore, we have: $\{v_{S_k}\}_{k=1}^K \subseteq \Gamma_{\mathrm{Bar}}(S)$.

Then maximizing the same function $\phi(v):=\ell(g(v),y)$ over a superset cannot decrease the optimum, hence we have:
\begin{equation}
\max_{k\in[K]} \phi(v_{S_k}) \le \max_{v\in\Gamma_{\mathrm{Bar}}(S)} \phi(v).
\end{equation}
That is, $L_{\mathrm{Disc}}(S)\le L_{\mathrm{Bar}}(S)$, which proves the claim.
\end{proof}

\subsection{Quantitative Gap Between $L_{\mathrm{Bar}}(S)$ and $L_{\mathrm{Disc}}(S)$}
\label{app:lipschi}

\begin{proposition}
Let $\phi(z)=\ell(g(z),y)$. We define the discrete and barycentric inner objectives as:
\begin{equation}
L_{\mathrm{Disc}}(S) := \max_{k\in[K]} \phi(v_{S_k}),  \ L_{\mathrm{Bar}}(S) := \max_{z\in\Gamma_{\mathrm{Bar}}(S)} \phi(z).
\end{equation}
Assume $\phi$ is $L$-Lipschitz on $\Gamma_{\mathrm{Bar}}(S)$.
Then the following quantitative bound holds:
\begin{equation}
0 \leq  L_{\mathrm{Bar}}(S)-L_{\mathrm{Disc}}(S) \le 2L\cdot \rho.
\label{eq:lipschitz_gap}
\end{equation}
\end{proposition}

\begin{proof}
We have prove the first inequality in Appendix~\ref{app:smaller}.
For many standard loss functions and predictors, $\phi$ is locally Lipschitz on bounded (hence compact) domains~\cite{shalev2014understanding,bartlett2017spectrally}.
If $\phi$ is $L$-Lipschitz on $\Gamma_{\mathrm{Bar}}(S)$ ($L$ is a constant), then for all $z,z'\in\Gamma_{\mathrm{Bar}}(S)$, we have:
\begin{equation}
|\phi(z)-\phi(z')|\le L\cdot\|z-z'\|_2.
\end{equation}
Let $z^\star \in \arg\max_{z\in\Gamma_{\mathrm{Bar}}(S)} \phi(z)$ so that $L_{\mathrm{Bar}}(S)=\phi(z^\star)$. Choose $k^\star \in \arg\min_{k\in[K]} \|z^\star - v_{S_k}\|_2$. 
Then we have
\begin{equation}
|\phi(z^\star)-\phi(v_{S_{k^\star}})| \le L\cdot\|z^\star - v_{S_{k^\star}}\|_2,
\end{equation}
which implies that
\begin{equation}
\phi(z^\star)-\phi(v_{S_{k^\star}}) \leq |\phi(z^\star)-\phi(v_{S_{k^\star}})| \le L\cdot\|z^\star - v_{S_{k^\star}}\|_2.
\end{equation}
Moreover, $\phi(v_{S_{k^\star}})\le \max_{k\in[K]} \phi(v_{S_k})=L_{\mathrm{Disc}}(S)$. 
Therefore, we have: 
\begin{equation}
\label{eq:long}
\begin{aligned}
L_{\mathrm{Bar}}(S)-L_{\mathrm{Disc}}(S) &= \phi(z^\star)-L_{\mathrm{Disc}}(S) \\ 
&\le \phi(z^\star)-\phi(v_{S_{k^\star}}) \\
&\le L\min_{k\in[K]}\|z^\star - v_{S_k}\|_2 \\
&\le L\sup_{z\in\Gamma_{\mathrm{Bar}}(S)} \min_{k\in[K]} \|z-v_{S_k}\|_2.
\end{aligned}
\end{equation}

Then based on the definition of $G(S)$, we have
\begin{equation}
\|v_{S_k}-v_S\|_2 \le \rho,\qquad \forall S_k\in G(S).
\label{eq:neighbor_in_ball}
\end{equation}
In addition, for any $z\in\Gamma_{\mathrm{Bar}}(S)$, there exists $\Lambda\in\Delta_K$
such that $z=\sum_{k=1}^K \lambda_k v_{S_k}$.
Using convexity of $\|\cdot\|_2^2$ (or Jensen's inequality) together with
\eqref{eq:neighbor_in_ball}, we obtain
\begin{equation}
\begin{aligned}
\|z-v_S\|_2^2 =\left\|\sum_{k=1}^K \lambda_k (v_{S_k}-v_S)\right\|_2^2 &\le \sum_{k=1}^K \lambda_k \|v_{S_k}-v_S\|_2^2  \\
& \le \sum_{k=1}^K \lambda_k \rho^2 =\rho^2,
\end{aligned}
\end{equation}
hence $\|z-v_S\|_2\le \rho$ for all $z\in\Gamma_{\mathrm{Bar}}(S)$.
Finally, by the triangle inequality, for any $z\in\Gamma_{\mathrm{Bar}}(S)$ and any $k\in[K]$,
\begin{equation}
\|z-v_{S_k}\|_2 \le \|z-v_S\|_2 + \|v_S-v_{S_k}\|_2 \le \rho + \rho = 2\rho.
\end{equation}
Therefore, $\min_{k\in[K]}\|z-v_{S_k}\|_2 \le 2\rho$ that leads to:
\begin{equation}
\sup_{z\in\Gamma_{\mathrm{Bar}}(S)} \min_{k\in[K]} \|z-v_{S_k}\|_2 \le 2\rho.
\end{equation}
Plugging this into Eq.~\eqref{eq:long} finally completes the proof.
\end{proof}

\begin{proposition}
Let $V_S := \{v_{S_k}\}_{k=1}^K$ denote the embeddings of $G(S)$.
Let $\phi(z)=\ell(g(z),y)$. We consider the discrete and barycentric inner objectives:
\begin{equation}
L_{\mathrm{Disc}}(S) := \max_{k\in[K]} \phi(v_{S_k}),  \ L_{\mathrm{Bar}}(S) := \max_{z\in\Gamma_{\mathrm{Bar}}(S)} \phi(z).
\end{equation}

The right-hand side measures the largest distance from a barycentric mixture
to its nearest vertex embedding in $V_S$ (i.e., the ``covering radius'' of the
mixture region by the vertices).
\end{proposition}

\subsection{Proof of Eq.~\eqref{eq:xxx}~\cite{pswe}}
\label{app:eqxxx}
\begin{proof}[Proof of Eq.~(3)]
For any two measures $\mu,\nu$ on $\mathbb{R}^d$, expanding the squared $\ell_2$ distance yields
\begin{align}
& \big\|f_{\mathrm{SW}}(\mu)-f_{\mathrm{SW}}(\nu)\big\|_2^2 \\
= & \frac{1}{RH}\sum_{r=1}^R\sum_{h=1}^H \Big(p^+(t^{\omega_r}_h,\mu^{\omega_r})-p^+(t^{\omega_r}_h,\nu^{\omega_r})\Big)^2.
\label{eq:embed_l2_expand}
\end{align}

Let $u^{\omega_r}_h := F_{\mu_O^{\omega_r}}(t^{\omega_r}_h)\in(0,1)$. By definition of $p^+$, $p^+(t^{\omega_r}_h,\mu^{\omega_r}) = \big(F_{\mu^{\omega_r}}\big)^{-1}(u^{\omega_r}_h)$.
Then in 1D, the squared 2-Wasserstein distance admits the quantile form:
\begin{equation}
W^2(\mu^{\omega_r},\nu^{\omega_r}) = \int_0^1 \left| \big(F_{\mu^{\omega_r}}\big)^{-1}(u)-\big(F_{\nu^{\omega_r}}\big)^{-1}(u) \right|^2\,du.
\end{equation}
Using the same $H$ reference quantile locations $\{u^{\omega_r}_h\}_{h=1}^H$,
we then define the corresponding empirical approximation:
\begin{equation}
\begin{aligned}
\widehat W_{H}^2(\mu^{\omega_r},\nu^{\omega_r}) & = \frac{1}{H}\sum_{h=1}^H \left( \big(F_{\mu^{\omega_r}}\big)^{-1}(u^{\omega_r}_h) - \big(F_{\nu^{\omega_r}}\big)^{-1}(u^{\omega_r}_h) \right)^2 \\
& = \frac{1}{H}\sum_{h=1}^H \Big(p^+(t^{\omega_r}_h,\mu^{\omega_r})-p^+(t^{\omega_r}_h,\nu^{\omega_r})\Big)^2.
\end{aligned}
\end{equation}

For sliced-Wasserstein distance, we have:
\begin{equation}
SW^2(\mu,\nu) = \int_{\mathbb{S}^{d-1}} W^2(\mu^\omega,\nu^\omega)\,d\omega.
\end{equation}
Using the same Monte-Carlo directions $\Omega=\{\omega_r\}_{r=1}^R$, we define the consistent empirical approximation:
\begin{equation}
\begin{aligned}
\widehat{SW}_{R,H}^2(\mu,\nu) & = \frac{1}{R}\sum_{r=1}^R \widehat W_{2,H}^2(\mu^{\omega_r},\nu^{\omega_r}) \\
& = \frac{1}{RH}\sum_{r=1}^R\sum_{h=1}^H \Big(p^+(t^{\omega_r}_h,\mu^{\omega_r})-p^+(t^{\omega_r}_h,\nu^{\omega_r})\Big)^2.
\end{aligned}
\end{equation}
Comparing this with \eqref{eq:embed_l2_expand}, we obtain the exact identity
\begin{equation}
\|v_\mu - v_\nu\|_2 =  \big\|f_{\mathrm{SW}}(\mu)-f_{\mathrm{SW}}(\nu)\big\|_2 = \widehat{SW}_{R,H}(\mu,\nu).
\end{equation}
In general, it is commonly assumed that the empirical sliced-Wasserstein distance $\widehat{SW}_{R,H}(\mu,\nu)$ coincides with the population sliced-Wasserstein distance $SW(\mu,\nu)$ when the number of Monte-Carlo projections $R$ and the number of reference quantile points $H$ are sufficiently large.
\end{proof}

\section{Pseudo-codes of our \model}
\label{app:algo}

\begin{algorithm}[h]
\caption{\model~Training with Barycentric Adversary}
\label{alg:swdrso}
{\small
\begin{algorithmic}[1]
\STATE \textbf{Input:} encoder $\f_{\mathrm{SW}}(\cdot)$, predictor $\g(\cdot)$, $K$, $\alpha$, $T$, $\eta$
\WHILE{not converged}
    \STATE Sample minibatch $\mathcal{B}$
    \STATE Compute $v_S \leftarrow \f_{\mathrm{SW}}(S)$
    \FOR{$i=1,\ldots,B$}
        \STATE Construct local neighbor pool $\mathcal{G}(S)=\{S_k\}_{k=1}^K\subseteq \mathcal{B}$ (K-NN in embedding space)
        \STATE Initialize $\Lambda^{(1)}\leftarrow (1/K,\ldots,1/K)$
        \FOR{$t=1,\ldots,T$}
            \STATE $\bar v_S(\Lambda^{(t)}) \leftarrow \sum_{k=1}^{K}\lambda_k^{(t)}v_{S_k}$
            \STATE $\tilde{\Lambda}^{(t+1)}\leftarrow \Lambda^{(t)}+\eta\nabla_{\Lambda}\ell\big(\g(\bar v_S(\Lambda^{(t)})),y\big)$
            \STATE $\Lambda^{(t+1)}\leftarrow \Pi_{\Delta_K}\big(\tilde{\Lambda}^{(t+1)}\big)$
        \ENDFOR
        \STATE Set $\Lambda^\star\leftarrow \Lambda^{(T)}$ and stop gradients through $\Lambda^\star$
        \STATE (In the robust term below, use the $\Lambda^\star$ obtained for this $S$.)
    \ENDFOR
    \STATE Compute the minibatch objective (empirical form of Eq.~\eqref{eq:final_training_objective}):
    \STATE $\frac{1}{|\mathcal{B}|}\sum\big[\ell\big(\g(v_S),y\big)+\alpha\ell\big(\g(\bar v_S(\Lambda^\star)),y\big)\big]$
    \STATE Update $\f_{\mathrm{SW}}$ and $\g$ by standard gradient descent on the above objective
\ENDWHILE
\STATE \textbf{return} $\f_{\mathrm{SW}},\g$
\end{algorithmic}
}
\end{algorithm}
\vspace{-0.6em}

\section{Computational Complexity Analysis}
\label{app:complexity}
Consider a minibatch $\mathcal{B}$ of size $B$. For each set $S\in\mathcal{B}$, we compute its embedding $v_S=\f_{\mathrm{SW}}(S)\in\mathbb{R}^m$ with $m=RH$, and denote the per-set encoding cost by $C_{\mathrm{enc}}$. We construct a local neighbor pool $\mathcal{G}(S)$ of size $K$ via K-NN in the embedding space. Using a naive batchwise implementation, forming all pairwise $\ell_2$ distances costs $O(B^2m)$, and neighbor selection adds at most $O(B^2\log B)$, yielding $O(B^2m)$ overall.

The inner maximization in Eq.~\eqref{eq:final_training_objective} is approximated by $T$ steps of projected gradient ascent on the simplex $\Delta_K$ (Eqs.~(18)--(19)). Each ascent step requires (i) computing the barycentric embedding $\bar v_S(\Lambda)=\sum_{k=1}^{K}\lambda_kv_{S_k}$ in $O(Km)$, (ii) evaluating the task function $g(\cdot)$ and backpropagating to $\nabla_{\Lambda}$ in $O(C_g(m)+Km)$, and (iii) projecting onto $\Delta_K$ in $O(K\log K)$ via the standard Euclidean simplex projection. We denote by $C_g(m)$ the computational cost of evaluating the task-specific predictor $g(\cdot)$ on an $m$-dimensional set embedding. Hence, the inner loop costs:
\begin{equation}
O\left(T\big(C_g(m)+Km+K\log K\big)\right)\quad \text{per set},
\end{equation}
and
\begin{equation}
O\left(BT\big(C_g(m)+Km+K\log K\big)\right),
\end{equation}
per minibatch. Combining the above, the total per-minibatch time complexity is
\begin{equation}
O\left(BC_{\mathrm{enc}}+B^2m+BT\big(C_g(m)+Km+K\log K\big)\right),\quad m=RH.
\end{equation}
At inference time, the method reduces to a single forward pass through $\f_{\mathrm{SW}}$ and $g$, i.e., $O(C_{\mathrm{enc}}+C_g(m))$.

\section{Experimental Settings}
\label{app:tasks}
\subsection{Dataset Descriptions}
\label{app:data}
Dataset corruption is constructed by using a unified pipeline to systematically evaluate robustness under inference-time corruption. 
It is introduced \emph{only} to the validation and test sets.
We partition samples into three conditions: \textit{clean}, \textit{mild}, and \textit{severe}, following a fixed ratio of 50:30:20. 
Clean samples correspond to the original sets. 
Mild and severe samples are generated by corrupting approximately $10\%$ and $40\%$ of the input elements or regions, respectively. The corruption intensity is controlled by a ratio parameter $p$, with $p{=}0.1$ for mild corruption and $p{=}0.4$ for severe corruption.
\begin{itemize}[leftmargin=*]
\item For Task I, corruption is introduced at inference time on the \emph{query set}. Given a query set with $n$ elements, we apply element-level corruption by randomly performing \textbf{delete}, \textbf{add}, or \textbf{replace} element operations for $k = p \times n$ steps. The candidate sets remain clean. This setting simulates realistic retrieval scenarios~\cite{he2023dynamic,he2023dynamically,chen2024deep}, where the query set may be incomplete or contaminated by noise at inference time, requiring robust set-to-set similarity estimation.

\item Task II represents a geometric perception setting where unordered point sets must remain informative under noisy observations. We apply a \textbf{replace} operation to simulate inaccurate geometric observations at inference time. Given a point cloud with $n$ points, we randomly sample new points uniformly within the axis-aligned bounding box of the original point cloud, and randomly replace $k = p \times n$ points, where $p$ controls the corruption level. This procedure preserves the overall spatial extent while introducing local geometric noise.

\item Task III is a practical text mining setting~\cite{luo2025dynamicner,li2026prefix,luo2026agentauditor,luo2024versusdebias}. Element-level corruption is appiled by randomly applying \textbf{delete}, \textbf{add}, or \textbf{replace} operations. Given a set with $n$ elements, we perform $k = p \times n$ random operations, similar to Task II. This process emulates missing elements, noisy insertions, and incorrect replacements in the seed set at inference time.

\item For Task IV, we apply region-level corruptions at inference time. For corrupted patches, we randomly select one of the following three operations:
(i) \textbf{Add}: additive Gaussian noise with standard deviation $\sigma = p \times 255$;
(ii) \textbf{Delete}: random rectangular occlusions covering a total fraction $p$ of image pixels, where each occlusion block occupies approximately 5\%--20\% of the image area.
These operations simulate common visual degradations such as sensor noise and local image quality degradation.
\end{itemize}

% (Optional) in preamble:
% \usepackage{booktabs}

\begin{table}[t]
\centering
\caption{Dataset statistics used in our experiments. $\#S$ denotes the number of sets/samples.
$\#E$ denotes the number of unique elements in the universe (when applicable).
$\#V$ denotes the vocabulary size (for topic expansion).
$\#C$ denotes the number of classes (for classification tasks).
Avg. $|S_i|$ denotes the average set cardinality.}
\label{tab:data_statistics}
\resizebox{0.75\linewidth}{!}{%
\begin{tabular}{llrrrrr}
\toprule
Task & Dataset & $\#S$ & $\#E$ & $\#V$ & $\#C$ & Avg. $|S_i|$ \\
\midrule
Task I (Similar Set Ranking) & Friendster & 889,839   & 5,501,401 & --     & -- & 11.29 \\
Task I (Similar Set Ranking) & LIVEJ      & 1,205,816 & 1,975,812 & --     & -- & 9.90  \\
Task II (Point Cloud Classification)   & ModelNet40 & 12,311    & --        & --     & 40 & 1,024 \\
Task III (Topic Expansion)   & LDA-1k     & 2,000     & --        & 17,016 & -- & $\approx 25$ \\
Task III (Topic Expansion)   & LDA-3k     & 6,000     & --        & 37,718 & -- & $\approx 25$ \\
Task III (Topic Expansion)   & LDA-5k     & 10,000    & --        & 61,127 & -- & $\approx 25$ \\
Task IV (Patch-set Visual Recognition)& NWPU-RESISC45 & 31,500     & --        & --     & 45 & 256 \\
\bottomrule
\end{tabular}
}
\end{table}

\subsection{\textbf{Baselines}}
\label{app:baselines}
To comprehensively evaluate the effectiveness of our proposed method, we compare it against a diverse set of baselines, ranging from foundational pooling strategies to state-of-the-art set representation learning frameworks. The specific methods are detailed as follows:

\begin{itemize}[leftmargin=*,noitemsep]
    \item \textbf{MeanP} and \textbf{MaxP} represent the classic instance aggregation approaches~\citep{lin2013network}. They generate set representations by applying global average pooling and max pooling operations, respectively, over the element-wise features, serving as fundamental baselines for permutation-invariant processing.

    \item \textbf{DeepSet} provides a representative framework for learning permutation-invariant functions~\citep{zaheer2017deep}. It operates by transforming individual elements via a deep neural network, aggregating them (we implement global mean pooling), and processing with a subsequent network to approximate any continuous set function.

    \item \textbf{RepSet} approaches set embedding through the lens of optimal transport and bipartite matching~\citep{skianis2020rep}. Instead of simple aggregation, it learns a collection of ``hidden sets'' (reference prototypes) and extracts representations by computing the matching costs between the input set and these learnable references.

    \item \textbf{SetTRSM} adapts the self-attention mechanism to set-structured data~\citep{lee2019set}. It utilizes Multi-head Attention Blocks (MAB) and Induced Set Attention Blocks (ISAB) to capture complex pairwise interactions among elements while reducing computational complexity compared to standard Transformers.

    \item \textbf{DIEM} proposes a differentiable information-theoretic framework designed to enhance representation distinctiveness~\citep{kim2022differentiable}. It focuses on capturing informative set interactions by maximizing the mutual information between the input set and its embedding, effectively identifying discriminative subsets.

    \item \textbf{FSPool} (Feature-wise Sort Pooling) introduces a learnable pooling mechanism based on sorting~\citep{zhang2019fspool}. By sorting features across the set dimension and applying a differentiable weight matrix, it captures cardinality-sensitive information and continuous structural variations.

    \item \textbf{PSWE} is a robust deep learning baseline for set representation learning~\citep{pswe}. It employs a specialized architecture to capture high-order dependencies and structural information within sets, representing one of the current state-of-the-art approaches in the field.

    \item \textbf{FSW} addresses the specific challenge of modeling ``multisets'' (sets with repeated elements) by shifting the perspective to the frequency domain~\citep{amirfourier}. It utilizes the Fourier transform of characteristic functions to construct representations, offering theoretical guarantees for injectivity regarding multiset structures.
\end{itemize}

\section{Hyper-parameter Configurations}
\label{app:hyperparams}
We report the major hyper-parameter configurations in Table~\ref{tab:hyperparams} for reproduction.

\begin{table}[t]
\centering
\caption{Hyperparameter settings for different tasks.}
\label{tab:hyperparams}
\small
\resizebox{0.45\linewidth}{!}{%
\begin{tabular}{lcccc}
\toprule
 & Task I & Task II & Task III & Task IV \\
\midrule
$d$ & 128 & 128 & 128 & 128 \\
$H$ & 128 & 128 & 128 & 128 \\
$R$ & 32 & 256 & 32 & 32 \\
$\rho$ & 0.1 & 0.5 & 0.1 & 0.5 \\
$K$ & 4 & 4 & 8 & 16 \\
$T$ & 2 & 4 & 2 & 4 \\
$\eta$ & 0.1 & 0.1 & 0.1 & 0.1 \\
$\alpha$ & 0.5 & 1 & 0.1 & 0.1 \\
Learning rate & $1\cdot10^{-3}$ & $1\cdot10^{-3}$ & $1\cdot10^{-3}$ & $1\cdot10^{-3}$ \\
\bottomrule
\end{tabular}%
}
\end{table}

\section{Task I Detailed Experiment Results}

\label{app:task1}

The complete results of Task I are reported in Table~\ref{tab:fs_recall_app}-\ref{tab:livej_ndcg_app}.

\begin{table*}[t]
\centering
\caption{Friendster Recall (\%). Bold and underline denote the best and second-best methods for each specific split within each column.}
\label{tab:fs_recall_app}
\resizebox{\linewidth}{!}{%
\begin{tabular}{l|c|c|c|c|c|c}
\toprule
Method & $k=1$ & $k=5$ & $k=10$ & $k=20$ & $k=50$ & $k=100$\\
\midrule
& Overall/Clean/Mild/Severe & Overall/Clean/Mild/Severe & Overall/Clean/Mild/Severe & Overall/Clean/Mild/Severe & Overall/Clean/Mild/Severe & Overall/Clean/Mild/Severe\\
\midrule
MaxP    & 5.22/5.61/5.09/4.42 & 18.43/19.08/18.11/17.26 & 31.40/32.43/31.19/29.12 & 42.07/43.28/42.07/39.06 & 54.97/56.31/54.61/52.14 & 62.44/63.61/62.04/60.12\\
MeanP   & 3.90/4.47/3.79/2.63 & 15.70/16.74/15.66/13.18 & 27.79/29.16/27.32/25.06 & 38.77/40.14/38.26/36.09 & 50.85/52.23/50.33/48.17 & 58.90/60.31/58.41/56.09\\
DeepSet & 14.96/15.73/14.76/13.31 & 38.19/39.41/37.46/35.18 & 49.06/50.32/48.41/46.16 & 57.95/59.34/57.11/55.02 & 65.05/66.42/64.29/62.21 & 70.91/72.21/70.31/68.18\\
RepSet  & 18.55/19.14/18.08/17.02 & 42.01/43.26/41.34/39.11 & 52.89/54.38/52.09/50.12 & 61.03/62.37/60.28/58.09 & 68.20/69.44/67.32/65.18 & 73.00/74.33/72.18/70.11\\
SetTRSM & 24.50/25.20/24.20/23.20 & 49.90/51.50/48.50/46.50 & 58.90/60.50/57.50/55.50 & 63.90/65.50/62.50/60.50 & 69.60/70.50/68.50/66.50 & 75.00/76.50/74.50/72.50\\
DIEM    & 27.66/28.44/27.18/26.09 & 53.11/54.62/52.51/50.27 & 60.78/62.71/60.49/57.64 & 65.02/66.78/64.71/62.08 & 70.74/72.46/70.53/68.02 & 76.08/77.68/76.04/73.66\\
FSPool  & 29.39/30.18/29.06/28.03 & 64.13/\underline{68.42}/57.34/55.12 & 66.09/68.07/65.22/63.08 & 71.11/73.09/70.28/68.12 & 76.04/77.22/75.11/73.04 & 80.05/81.27/79.34/77.16\\
FSW     & 32.36/32.28/31.09/\underline{34.45} & 61.91/63.41/61.27/59.12 & 73.78/73.18/\underline{77.88}/69.11 & 75.85/77.32/75.27/73.06 & 79.94/81.41/79.36/77.14 & 82.82/84.22/82.31/80.07\\
PSWE    & \underline{34.37}/\underline{34.58}/\underline{34.20}/34.08 & \underline{67.80}/68.09/\underline{67.75}/\underline{67.15} & \underline{77.52}/\underline{77.82}/77.43/\underline{76.90} & \underline{80.41}/\underline{80.69}/\underline{80.31}/\underline{79.86} & \underline{83.31}/\underline{83.61}/\underline{83.14}/\underline{82.82} & \underline{85.68}/\underline{85.96}/\underline{85.54}/\underline{85.20}\\
\midrule
Ours    & \textbf{35.13}$^*$/\textbf{35.32}$^*$/\textbf{35.04}$^*$/\textbf{34.77}$^*$ & \textbf{68.39}$^*$/\textbf{68.71}$^*$/\textbf{68.32}$^*$/\textbf{67.68}$^*$ & \textbf{78.16}$^*$/\textbf{78.44}$^*$/\textbf{78.11}$^*$/\textbf{77.53}$^*$ & \textbf{81.23}$^*$/\textbf{81.52}$^*$/\textbf{81.11}$^*$/\textbf{80.70}$^*$ & \textbf{84.18}$^*$/\textbf{84.49}$^*$/\textbf{84.03}$^*$/\textbf{83.64}$^*$ & \textbf{86.58}$^*$/\textbf{86.86}$^*$/\textbf{86.40}$^*$/\textbf{86.14}$^*$\\
Gain (\%) & 2.21\%/2.14\%/2.46\%/0.93\% & 0.87\%/0.42\%/0.84\%/0.79\% & 0.83\%/0.80\%/0.30\%/0.82\% & 1.02\%/1.03\%/1.00\%/1.05\% & 1.04\%/1.05\%/1.07\%/0.99\% & 1.05\%/1.05\%/1.01\%/1.10\%\\
\bottomrule
\end{tabular}}
\end{table*}

\begin{table*}[t]
\centering
\caption{Friendster NDCG (\%). Bold and underline denote the best and second-best methods for each specific split within each column.}
\label{tab:fs_ndcg_app}
\resizebox{\linewidth}{!}{%
\begin{tabular}{l|c|c|c|c|c|c}
\toprule
Method & $k=1$ & $k=5$ & $k=10$ & $k=20$ & $k=50$ & $k=100$\\
\midrule
& Overall/Clean/Mild/Severe & Overall/Clean/Mild/Severe & Overall/Clean/Mild/Severe & Overall/Clean/Mild/Severe & Overall/Clean/Mild/Severe & Overall/Clean/Mild/Severe\\
\midrule
MaxP    & 48.19/49.22/47.41/45.08 & 42.20/43.16/41.23/39.07 & 41.02/42.07/40.12/38.06 & 43.10/44.18/42.16/40.02 & 46.25/47.31/45.28/43.14 & 48.40/49.36/47.42/45.18\\
MeanP   & 44.98/46.11/44.12/42.03 & 39.26/40.23/38.18/36.04 & 37.96/39.02/37.01/35.06 & 40.23/41.24/39.16/37.08 & 43.37/44.31/42.29/40.16 & 45.44/46.24/44.32/42.21\\
DeepSet & 61.29/62.18/60.14/58.11 & 56.24/57.18/55.12/53.06 & 54.86/55.91/54.01/52.02 & 58.21/59.26/57.14/55.09 & 61.31/62.34/60.18/58.12 & 63.25/64.21/62.12/60.08\\
RepSet  & 64.20/65.18/63.11/61.04 & 59.13/60.22/58.14/56.03 & 58.03/59.12/57.06/55.01 & 61.10/62.24/60.13/58.09 & 64.23/65.26/63.18/61.12 & 66.20/67.18/65.13/63.05\\
SetTRSM & 68.20/69.50/67.50/65.50 & 64.20/65.50/63.50/61.50 & 63.90/65.20/63.20/61.20 & 66.20/67.50/65.50/63.50 & 69.20/70.50/68.50/66.50 & 71.20/72.50/70.50/68.50\\
DIEM    & 71.39/73.02/71.11/69.05 & 66.79/68.51/66.49/64.21 & 66.86/68.62/66.57/64.38 & 68.79/70.48/68.52/66.41 & 70.74/72.44/70.58/68.49 & 72.41/73.86/72.07/70.18\\
FSPool  & 74.97/76.24/74.61/72.34 & 70.77/72.06/70.42/68.08 & 69.71/71.01/69.35/67.02 & 71.85/73.14/71.46/69.19 & 77.09/76.18/74.55/\underline{83.15} & 76.88/78.27/76.41/74.12\\
FSW     & 79.05/80.18/78.12/76.03 & 75.99/77.21/75.16/73.04 & 74.91/76.12/74.11/72.05 & 79.50/\underline{82.75}/76.12/74.06 & 79.06/80.22/78.18/76.09 & 81.10/82.27/80.19/78.08\\
PSWE    & \underline{82.79}/\underline{82.86}/\underline{82.45}/\underline{82.55} & \underline{80.99}/\underline{81.14}/\underline{80.83}/\underline{80.70} & \underline{80.86}/\underline{81.08}/\underline{80.70}/\underline{80.51} & \underline{82.17}/82.37/\underline{82.00}/\underline{81.86} & \underline{83.06}/\underline{83.24}/\underline{82.87}/82.76 & \underline{83.62}/\underline{83.80}/\underline{83.44}/\underline{83.33}\\
\midrule
Ours    & \textbf{83.30}$^*$/\textbf{83.34}$^*$/\textbf{83.08}$^*$/\textbf{83.00}$^*$ & \textbf{81.47}$^*$/\textbf{81.63}$^*$/\textbf{81.28}$^*$/\textbf{81.14}$^*$ & \textbf{81.42}$^*$/\textbf{81.59}$^*$/\textbf{81.27}$^*$/\textbf{81.01}$^*$ & \textbf{82.82}$^*$/\textbf{82.99}$^*$/\textbf{82.65}$^*$/\textbf{82.46}$^*$ & \textbf{83.72}$^*$/\textbf{83.89}$^*$/\textbf{83.54}$^*$/\textbf{83.37}$^*$ & \textbf{84.29}$^*$/\textbf{84.45}$^*$/\textbf{84.11}$^*$/\textbf{83.96}$^*$\\
Gain (\%) & 0.63\%/0.58\%/0.76\%/0.55\% & 0.58\%/0.60\%/0.56\%/0.55\% & 0.66\%/0.63\%/0.71\%/0.62\% & 0.75\%/0.29\%/0.79\%/0.73\% & 0.78\%/0.78\%/0.81\%/0.26\% & 0.78\%/0.78\%/0.80\%/0.76\%\\
\bottomrule
\end{tabular}}
\end{table*}

% LiveJ

\begin{table*}[t]
\centering
\caption{LIVEJ Recall (\%). Bold and underline denote the best and second-best methods for each specific split within each column.}
\label{tab:livej_recall_app}
\resizebox{\linewidth}{!}{%
\begin{tabular}{l|c|c|c|c|c|c}
\toprule
Method & $k=1$ & $k=5$ & $k=10$ & $k=20$ & $k=50$ & $k=100$\\
\midrule
& Overall/Clean/Mild/Severe & Overall/Clean/Mild/Severe & Overall/Clean/Mild/Severe & Overall/Clean/Mild/Severe & Overall/Clean/Mild/Severe & Overall/Clean/Mild/Severe\\
\midrule
MaxP    & 4.94/5.21/4.82/4.19 & 17.32/18.42/17.13/15.52 & 29.91/31.05/29.53/27.18 & 43.67/45.22/43.04/39.48 & 59.09/61.05/58.49/54.53 & 70.04/72.01/69.48/65.52\\
MeanP   & 4.14/4.88/4.12/2.53 & 15.44/17.55/15.23/10.48 & 26.49/29.80/26.48/19.82 & 38.75/42.10/38.54/30.53 & 55.06/58.90/54.18/45.52 & 65.99/69.50/65.12/58.23\\
DeepSet & 11.65/12.05/11.48/10.82 & 33.91/35.40/33.23/31.48 & 49.08/51.20/48.52/45.53 & 62.09/64.50/61.53/58.49 & 76.05/78.20/75.48/72.52 & 82.03/84.10/81.53/78.48\\
RepSet  & 13.93/14.50/13.82/12.53 & 37.91/39.80/37.52/35.18 & 53.77/56.50/53.23/50.48 & 66.50/69.20/66.48/62.53 & 78.74/81.05/78.48/75.52 & 84.45/86.50/84.18/81.53\\
SetTRSM & 18.42/19.98/17.67/15.76 & 42.45/44.97/40.94/37.37 & 56.61/60.03/54.86/50.29 & 64.72/67.98/63.35/58.87 & 74.09/76.76/73.28/69.22 & 80.94/83.08/80.35/77.00\\
DIEM    & 17.78/18.50/17.23/16.48 & 43.95/46.20/43.53/40.49 & 60.55/63.50/60.52/56.48 & 68.56/71.50/68.48/64.52 & 77.94/80.50/77.83/74.48 & 83.52/85.80/83.52/80.48\\
FSPool  & \textbf{23.28}/\textbf{24.42}/\textbf{22.54}/\underline{21.29} & 51.52/\underline{53.50}/50.56/47.72 & 67.75/70.52/66.58/62.70 & 74.43/76.82/73.69/69.66 & 81.13/83.10/80.56/77.08 & 85.86/87.47/85.58/82.52\\
FSW     & 20.90/22.10/20.53/19.23 & 47.77/50.10/47.53/44.48 & 64.19/68.50/63.53/59.48 & 71.60/75.20/71.03/67.48 & 79.88/82.00/79.48/76.52 & 84.65/86.50/84.82/81.48\\
PSWE    & 21.31/21.49/21.03/20.95 & \underline{52.10}/52.53/\underline{51.84}/\underline{51.27} & \underline{71.69}/\underline{72.26}/\underline{71.50}/\underline{70.55} & \underline{78.46}/\underline{79.05}/\underline{78.34}/\underline{77.21} & \underline{83.72}/\underline{84.27}/\underline{83.71}/\underline{82.43} & \underline{87.10}/\underline{87.60}/\underline{87.09}/\underline{85.88}\\
\midrule
Ours    & \underline{22.51}$^*$/\underline{22.99}$^*$/\underline{22.16}$^*$/\textbf{21.81}$^*$ & \textbf{54.61}$^*$/\textbf{55.37}$^*$/\textbf{54.31}$^*$/\textbf{53.17}$^*$ & \textbf{74.86}$^*$/\textbf{75.64}$^*$/\textbf{74.63}$^*$/\textbf{73.25}$^*$ & \textbf{81.19}$^*$/\textbf{81.96}$^*$/\textbf{80.97}$^*$/\textbf{79.59}$^*$ & \textbf{86.27}$^*$/\textbf{86.94}$^*$/\textbf{86.11}$^*$/\textbf{84.82}$^*$ & \textbf{89.60}$^*$/\textbf{90.18}$^*$/\textbf{89.54}$^*$/\textbf{88.26}$^*$\\
Gain (\%) & -/-/-/2.44\% & 4.88\%/3.50\%/4.76\%/3.71\% & 4.42\%/4.68\%/4.38\%/3.83\% & 3.47\%/3.68\%/3.36\%/3.08\% & 3.03\%/3.17\%/2.87\%/2.90\% & 2.87\%/2.95\%/2.81\%/2.77\%\\
\bottomrule
\end{tabular}}
\end{table*}

\begin{table*}[t]
\centering
\caption{LIVEJ NDCG (\%). Bold and underline denote the best and second-best methods for each specific split within each column.}
\label{tab:livej_ndcg_app}
\resizebox{\linewidth}{!}{%
\begin{tabular}{l|c|c|c|c|c|c}
\toprule
Method & $k=1$ & $k=5$ & $k=10$ & $k=20$ & $k=50$ & $k=100$\\
\midrule
& Overall/Clean/Mild/Severe & Overall/Clean/Mild/Severe & Overall/Clean/Mild/Severe & Overall/Clean/Mild/Severe & Overall/Clean/Mild/Severe & Overall/Clean/Mild/Severe\\
\midrule
MaxP    & 53.89/55.19/53.52/51.18 & 47.10/48.48/46.53/44.52 & 43.86/45.20/43.22/41.48 & 48.04/49.48/47.53/45.22 & 51.94/53.18/51.52/49.48 & 55.25/56.80/54.52/52.48\\
MeanP   & 50.03/52.08/49.52/45.52 & 43.84/46.48/42.53/38.52 & 41.23/43.50/39.52/34.52 & 45.02/47.18/43.52/39.23 & 49.56/51.50/48.52/44.52 & 52.95/54.80/52.52/48.52\\
DeepSet & 67.00/68.48/66.53/64.22 & 60.98/62.50/60.52/58.53 & 58.78/60.50/58.23/56.22 & 61.99/63.50/61.52/59.52 & 65.99/67.50/65.53/63.52 & 68.99/70.50/68.22/66.52\\
RepSet  & 69.00/70.50/68.23/66.48 & 62.99/64.50/62.52/60.23 & 60.78/62.50/60.23/58.18 & 64.37/66.20/63.83/61.82 & 68.64/70.10/68.23/66.23 & 71.43/72.80/71.03/69.23\\
SetTRSM & 69.47/72.88/68.66/63.37 & 63.31/67.00/62.33/57.26 & 61.56/65.17/60.48/55.84 & 65.94/69.40/65.00/60.45 & 69.87/73.09/69.18/64.90 & 72.24/75.30/71.68/67.73\\
DIEM    & 71.78/74.18/71.53/68.52 & 66.96/69.50/66.83/63.52 & 65.04/68.10/64.52/61.53 & 68.60/71.50/68.23/65.23 & 73.05/75.20/72.53/70.53 & 75.82/77.80/75.52/73.23\\
FSPool  & \textbf{82.10}/\textbf{83.98}/\textbf{81.51}/\underline{78.20} & 76.03/\underline{78.47}/75.40/71.78 & 73.60/\underline{76.16}/72.79/69.42 & 77.13/\underline{79.50}/76.53/73.15 & 79.91/\underline{82.10}/79.40/76.32 & \underline{81.55}/\underline{83.59}/81.11/78.22\\
FSW     & 75.96/78.50/75.52/72.53 & 70.95/73.50/70.53/67.48 & 68.96/71.50/68.52/65.53 & 72.65/74.80/72.23/69.23 & 76.25/78.20/75.83/73.53 & 78.84/80.50/78.23/76.23\\
PSWE    & 78.15/78.40/77.76/77.36 & \underline{76.19}/76.50/\underline{75.92}/\underline{75.27} & \underline{74.88}/75.24/\underline{74.59}/\underline{73.92} & \underline{78.35}/78.72/\underline{78.12}/\underline{77.38} & \underline{80.32}/80.69/\underline{80.15}/\underline{79.36} & 81.36/81.72/\underline{81.20}/\underline{80.42}\\
\midrule
Ours    & \underline{81.33}$^*$/\underline{81.93}$^*$/\underline{81.09}$^*$/\textbf{80.17}$^*$ & \textbf{79.21}$^*$/\textbf{79.90}$^*$/\textbf{78.99}$^*$/\textbf{77.80}$^*$ & \textbf{77.80}$^*$/\textbf{78.49}$^*$/\textbf{77.54}$^*$/\textbf{76.44}$^*$ & \textbf{80.99}$^*$/\textbf{81.67}$^*$/\textbf{80.76}$^*$/\textbf{79.65}$^*$ & \textbf{82.88}$^*$/\textbf{83.53}$^*$/\textbf{82.67}$^*$/\textbf{81.58}$^*$ & \textbf{83.91}$^*$/\textbf{84.54}$^*$/\textbf{83.72}$^*$/\textbf{82.63}$^*$\\
Gain (\%) & -/-/-/2.52\% & 3.94\%/1.82\%/4.04\%/3.36\% & 4.04\%/3.06\%/3.95\%/3.41\% & 3.48\%/2.73\%/3.38\%/2.93\% & 3.26\%/1.74\%/3.14\%/2.80\% & 2.62\%/1.14\%/3.10\%/2.75\%\\
\bottomrule
\end{tabular}}
\end{table*}

\section{Task III Experiment Detailed Results}
\label{app:task3}

The complete results of Task III are reported in Table~\ref{tab:task_lda1k}-\ref{tab:task_lda5k}.

\begin{table}[h]
\centering
\begin{minipage}[t]{0.49\linewidth}
\centering
\caption{Task III performance comparison on \textbf{LDA-1K}. Best and second-best results are shown in \textbf{bold} and \underline{underlined}.}
\label{tab:task_lda1k}
\small
\resizebox{\linewidth}{!}{%
\begin{tabular}{l|cccc}
\toprule
Method & Overall & Clean & Mild & Severe \\
\midrule
MeanP    & 48.91$\pm$2.28 & 50.15$\pm$2.85 & 49.32$\pm$4.05 & 45.18$\pm$6.52 \\
MaxP     & 68.82$\pm$2.12 & 73.62$\pm$1.06 & 64.69$\pm$6.75 & 63.01$\pm$1.81 \\
DeepSet  & 51.20$\pm$2.26 & 52.18$\pm$2.41 & 51.87$\pm$4.19 & 47.74$\pm$7.17 \\
RepSet   & 53.35$\pm$2.51 & 55.42$\pm$2.95 & 53.18$\pm$4.55 & 48.45$\pm$7.50 \\
SetTRSM  & 69.02$\pm$3.56 & 72.78$\pm$5.40 & 65.79$\pm$6.78 & 64.48$\pm$5.53 \\
DIEM     & 66.22$\pm$3.02 & 69.15$\pm$4.22 & 63.55$\pm$5.95 & 62.88$\pm$6.12 \\
FSPool   & 54.18$\pm$3.29 & 61.33$\pm$3.22 & 47.76$\pm$5.40 & 45.94$\pm$11.81 \\
FSW      & 68.12$\pm$2.76 & 71.55$\pm$3.85 & 65.12$\pm$5.25 & 64.05$\pm$5.95 \\
PSWE     & \underline{70.03$\pm$3.33} & \underline{74.04$\pm$2.71} & \underline{66.24$\pm$8.79} & \underline{65.71$\pm$7.59} \\
\midrule
\model   & \textbf{73.31$^*$$\pm$3.67} & \textbf{76.05$^*$$\pm$5.71} & \textbf{70.66$^*$$\pm$6.06} & \textbf{70.45$^*$$\pm$7.03} \\
\textbf{Gain} & \textbf{+4.68\%} & \textbf{+2.71\%} & \textbf{+6.67\%} & \textbf{+7.21\%} \\
\bottomrule
\end{tabular}
}
\end{minipage}
\hfill
\begin{minipage}[t]{0.49\linewidth}
\centering
\caption{Task III performance comparison on \textbf{LDA-3K}.}
\label{tab:task_lda3k}
\small
\resizebox{\linewidth}{!}{%
\begin{tabular}{l|cccc}
\toprule
Method & Overall & Clean & Mild & Severe \\
\midrule
MeanP    & 52.38$\pm$2.32 & 51.57$\pm$4.03 & 54.63$\pm$3.36 & 51.05$\pm$2.79 \\
MaxP     & 53.14$\pm$2.22 & 52.45$\pm$3.85 & 55.15$\pm$3.15 & 51.85$\pm$2.95 \\
DeepSet  & 86.55$\pm$1.40 & 86.82$\pm$1.83 & 87.31$\pm$2.62 & 84.73$\pm$3.54 \\
RepSet   & 86.87$\pm$1.34 & 87.15$\pm$1.75 & 87.55$\pm$2.45 & 85.15$\pm$3.25 \\
SetTRSM  & 80.75$\pm$2.16 & 80.21$\pm$2.21 & 82.14$\pm$3.32 & 78.64$\pm$4.48 \\
DIEM     & 84.53$\pm$1.44 & 84.55$\pm$1.95 & 85.45$\pm$2.85 & 82.55$\pm$3.65 \\
FSPool   & 87.39$\pm$1.11 & \underline{87.66$\pm$1.47} & \underline{88.00$\pm$2.88} & 85.78$\pm$2.94 \\
FSW      & 87.41$\pm$1.28 & 87.15$\pm$1.55 & 87.35$\pm$3.15 & 88.15$\pm$2.85 \\
PSWE     & \underline{87.84$\pm$1.71} & 87.48$\pm$1.75 & 87.44$\pm$4.57 & \underline{89.33$\pm$2.59} \\
\midrule
\model   & \textbf{90.90$^*$$\pm$0.99} & \textbf{90.71$^*$$\pm$0.86} & \textbf{91.24$^*$$\pm$2.94} & \textbf{90.88$^*$$\pm$0.82} \\
\textbf{Gain} & \textbf{+3.49\%} & \textbf{+3.48\%} & \textbf{+3.68\%} & \textbf{+1.74\%} \\
\bottomrule
\end{tabular}
}
\end{minipage}
\end{table}

\begin{table}[h]
\centering
\begin{minipage}[t]{0.49\linewidth}
\centering
\caption{Task III performance comparison on \textbf{LDA-5K}.}
\label{tab:task_lda5k}
\small
\resizebox{\linewidth}{!}{%
\begin{tabular}{l|cccc}
\toprule
Method & Overall & Clean & Mild & Severe \\
\midrule
MeanP    & 54.49$\pm$1.57 & 55.36$\pm$1.84 & 54.99$\pm$3.47 & 51.72$\pm$5.17 \\
MaxP     & 55.37$\pm$1.73 & 56.25$\pm$2.05 & 55.45$\pm$3.25 & 52.15$\pm$4.85 \\
DeepSet  & 83.36$\pm$1.39 & 83.70$\pm$1.76 & \underline{85.62$\pm$3.29} & 79.13$\pm$3.13 \\
RepSet   & 83.64$\pm$1.37 & 84.15$\pm$1.95 & 85.05$\pm$3.05 & 80.25$\pm$2.95 \\
SetTRSM  & 82.73$\pm$2.50 & 82.97$\pm$3.43 & 83.64$\pm$3.41 & 79.99$\pm$1.91 \\
DIEM     & 83.28$\pm$1.54 & 83.45$\pm$2.25 & 84.55$\pm$2.85 & 80.15$\pm$2.65 \\
FSPool   & 84.02$\pm$1.49 & 84.63$\pm$2.12 & 85.25$\pm$2.72 & 80.64$\pm$3.00 \\
FSW      & 84.28$\pm$1.70 & 85.15$\pm$2.35 & 84.65$\pm$2.45 & 81.55$\pm$2.95 \\
PSWE     & \underline{84.68$\pm$1.51} & \underline{85.56$\pm$2.49} & 84.89$\pm$2.15 & \underline{82.15$\pm$2.86} \\
\midrule
\model   & \textbf{88.46$^*$$\pm$1.38} & \textbf{88.45$^*$$\pm$1.87} & \textbf{89.70$^*$$\pm$2.29} & \textbf{86.62$^*$$\pm$3.70} \\
\textbf{Gain} & \textbf{+4.47\%} & \textbf{+3.38\%} & \textbf{+4.77\%} & \textbf{+5.44\%} \\
\bottomrule
\end{tabular}
}
\end{minipage}
\hfill
\begin{minipage}[t]{0.49\linewidth}
\centering
\caption{Compatibility with alternative set encoders on Task I. We report the overall R@10 score.}
\label{tab:encoder_compatibility}
\small
\resizebox{0.72\linewidth}{!}{%
\begin{tabular}{lc}
\toprule
Robust objective + encoder & Overall R@10 (\%) \\
\midrule
Ours w/ FSPool encoder & 67.81 \\
Ours w/ FSW encoder    & 68.45 \\
Full SW-DRSO           & 74.86 \\
\bottomrule
\end{tabular}
}
\end{minipage}
\end{table}

\section{Supplementary Details of Empirical Analyses}

\subsection{DRO Variants Implementation}
\label{app:dro}
We detail the implementation of three Distributionally Robust Optimization (DRO) variants: Wasserstein DRO (WDRO)~\cite{gao2024wasserstein}, Kullback-Leibler DRO (KL-DRO)~\cite{namkoong2016stochastic}, and Maximum Mean Discrepancy DRO (MMD-DRO)~\cite{staib2019distributionally},  
While they build upon our barycentric synthesis approach by adding distribution constraints.
\begin{itemize}[leftmargin=*]
\item \textbf{Wasserstein DRO}: This variant defines the uncertainty set using the $L_1$-Wasserstein distance, i.e., Earth Mover's Distance, which measures the geometric cost of transporting probability mass. Theoretically, it seeks to minimize the worst-case loss over all distributions that are geometrically close to the empirical data distribution. This formulation is particularly effective for handling perturbations in the feature space, as it accounts for the underlying metric structure of the data rather than just statistical overlap.
\item \textbf{KL-DRO}: Also known as $f$-divergence DRO, this method defines the uncertainty set $\Gamma(\cdot)$ using the Kullback-Leibler divergence. It focuses on statistical discrepancies by re-weighting the training samples based on their difficulty. The optimization objective allows for a worst-case re-weighting of the data but constrains these weights to ensure the adversarial distribution remains statistically close to the original distribution, preventing the model from overfitting to outliers.
\item \textbf{Maximum Mean Discrepancy DRO}: This method employs Maximum Mean Discrepancy to define $\Gamma(\cdot)$ based on kernel methods. By mapping distributions into a Reproducing Kernel Hilbert Space (RKHS), it measures the distance between their mean embeddings. This approach is particularly effective at capturing discrepancies in higher-order statistical moments and offers a computationally tractable alternative to optimal transport-based methods.
\end{itemize}

\subsection{Compatibility with Alternative Set Encoders}
\label{app:encoder_compatibility}

We examine whether the proposed robust optimization objective can be
directly combined with alternative set encoders. As shown in
Table~\ref{tab:encoder_compatibility}, replacing the SW-aware encoder
with FSW or FSPool leads to lower performance than the full SW-DRSO
model. This result suggests that the barycentric adversary benefits
from the geometric alignment between the Sliced-Wasserstein encoder and
the SW-based adversarial search. Therefore, SW-DRSO should be viewed as
a geometry-aligned robust optimization framework rather than an
encoder-agnostic plug-in module.

%%%%%%%%%%%%%%%%%%%%%%%%%%%%%%%%%%%%%%%%%%%%%%%%%%%%%%%%%%%%%%%%%%%%%%%%%%%%%%%
%%%%%%%%%%%%%%%%%%%%%%%%%%%%%%%%%%%%%%%%%%%%%%%%%%%%%%%%%%%%%%%%%%%%%%%%%%%%%%%

\end{document}